\documentclass{article}
\usepackage{microtype}
\usepackage{graphicx}
\usepackage{subfigure}
\usepackage{subcaption}
\usepackage{booktabs} 
\usepackage{bm}
\usepackage{hyperref}

\usepackage[accepted]{icml2024}

\usepackage{amsmath}
\usepackage{amssymb}
\usepackage{mathtools}
\usepackage{amsthm}
 \usepackage{multirow}
\usepackage[capitalize,noabbrev]{cleveref}

\theoremstyle{plain}
\newtheorem{theorem}{Theorem}[section]

\newtheorem{lemma}[theorem]{Lemma}

\theoremstyle{definition}

\theoremstyle{remark}

\def\vu{{\bm{u}}}

\def\vw{{\bm{w}}}
\def\vx{{\bm{x}}}

\newcommand{\E}{\mathbb{E}}
\usepackage[textsize=tiny]{todonotes}

\icmltitlerunning{Confidence-aware Contrastive Learning for Selective Classification}

\begin{document}

\twocolumn[
\icmltitle{Confidence-aware Contrastive Learning for Selective Classification}



\icmlsetsymbol{equal}{*}

\begin{icmlauthorlist}
\icmlauthor{Yu-Chang Wu}{yyy,comp}
\icmlauthor{Shen-Huan Lyu}{sch,zzz,yyy}
\icmlauthor{Haopu Shang}{yyy,comp}
\icmlauthor{Xiangyu Wang}{yyy,comp}
\icmlauthor{Chao Qian}{yyy,comp}
\end{icmlauthorlist}

\icmlaffiliation{yyy}{National Key Laboratory for Novel Software Technology, Nanjing University, China}
\icmlaffiliation{comp}{ School of Artificial Intelligence, Nanjing University, China}
\icmlaffiliation{zzz}{College of Computer Science and Software Engineering, Hohai University, China}
\icmlaffiliation{sch}{Key Laboratory of Water Big Data Technology of Ministry of Water Resources, Hohai University, China}

\icmlcorrespondingauthor{Chao Qian}{qianc@nju.edu.cn}

\icmlkeywords{Machine Learning, ICML}

\vskip 0.3in
]




\begin{abstract}
Selective classification enables models to make predictions only when they are sufficiently confident, aiming to enhance safety and reliability, which is important in high-stakes scenarios. Previous methods mainly use deep neural networks and focus on modifying the architecture of classification layers to enable the model to estimate the confidence of its prediction. This work provides a generalization bound for selective classification, disclosing that optimizing feature layers helps improve the performance of selective classification. Inspired by this theory, we propose to explicitly improve the selective classification model at the feature level for the first time, leading to a novel Confidence-aware Contrastive Learning method for Selective Classification, CCL-SC, which similarizes the features of homogeneous instances and differentiates the features of heterogeneous instances, with the strength controlled by the model's confidence. The experimental results on typical datasets, i.e., CIFAR-10, CIFAR-100, CelebA, and ImageNet, show that CCL-SC achieves significantly lower selective risk than state-of-the-art methods, across almost all coverage degrees. Moreover, it can be combined with existing methods to bring further improvement.
\end{abstract}
\printAffiliationsAndNotice{}  
\section{Introduction}

As Deep Neural Networks (DNNs) have been widely adopted across various industries, the reliability of their predictive outcomes has become increasingly critical. In many high-stakes domains such as medical diagnosis~\cite{Nature:medicalDiagnosis}, self-driving~\cite{self-driving}, or security systems~\cite{security-system}, erroneous predictions may lead to serious repercussions~\cite{AIsafty:ConcreteProblem}. The concept of selective classification for DNNs emerges in this context, providing a mechanism that allows a DNN to decide whether to make a prediction on an instance based on its prediction confidence estimation~\cite{Problem:El-YanivW10}. The goal of selective classification typically revolves around minimizing the model's selective risk while maintaining a high prediction coverage rate~\cite{GeifmanE17:coverage}.


The key issue in selective classification is how to select samples that may be predicted incorrectly and hand them over to humans for delayed prediction. A direct method is to use the maximum logit in the Softmax Layer (SR)~\cite{SR,GeifmanE17:coverage} of the model as the confidence function; a higher value indicates that the model is more confident in predicting the sample. 
Another approach is to utilize the inferences of multiple models to estimate the prediction confidence, such as \mbox{MC-dropout}~\cite{MC-dropout}, deep ensemble~\cite{DeepEnsemble}, and snapshot ensemble~\cite{snapshot-ensemble}. Given the expensive training or prediction costs, recent works predominantly focus on individual selective classification models. \mbox{SelectiveNet (SN)}~\cite{SelectiveNet} introduces an additional selection head to learn the confidence of predictions within a given coverage constraint. \mbox{Deep Gamblers (DG)}~\cite{DeepGambler} and \mbox{Self-Adaptive Training (SAT)}~\cite{NIPS:SAT,SAT} add a homogeneous logit to the output layer, serving as an ``abstention head" to predict the confidence of abstaining from making predictions.~\citet{Entropy+SR} proposed an additional Entropy-Minimization (EM) regularization loss to make the model more confident in its predictions, and applied it to the SAT method. However, their results suggest that the state-of-the-art selective classification methods with explicit selective heads actually lead to higher selective risk compared to directly using SR for confidence prediction. 

In this work, we provide a generalization bound for selective classification, disclosing that optimizing feature layers to reduce variance between samples of the same category is helpful for improving the performance. In addition, the selective classification problem inherently requires models to better differentiate between correctly classified and misclassified samples, and also requires consistency between the predictive confidence and the reliability of the classification results. Based on these analyses, we propose to improve the performance of selective classification by explicitly optimizing the feature representation of the model for the first time, instead of focusing on modifying the classification layer as in previous works. Specifically, we propose a novel Confidence-aware Contrastive Learning method for Selective Classification named CCL-SC, which aims to pull normalized feature embeddings from the same class that are correctly classified to be closer than embeddings that are misclassified as the same class. The ``pulling strength" is controlled by the model's confidence, i.e., the model pays more attention to samples with higher confidence during training, leading to a robust alignment between the model's predictive confidence and its actual accuracy.

We conduct experiments to compare our method CCL-SC with SR~\cite{SR,GeifmanE17:coverage}, DG~\cite{DeepGambler}, SAT~\cite{NIPS:SAT,SAT}, and SAT+EM~\cite{Entropy+SR} on four typical datasets including CIFAR-10, CIFAR-100, CelebA, and ImageNet. The results show that CCL-SC achieves significantly lower selective risk than these SOTA methods across almost all degrees of coverage. The t-SNE visualization clearly shows that CCL-SC achieves better feature representation, i.e., significant intra-class aggregation and inter-class separation in the embedded feature space. We also perform comprehensive sensitivity analyses of the hyper-parameters, demonstrating the robustness of CCL-SC, and the alignment between our our proposed method and theory. It is noteworthy that our method CCL-SC optimizes the model from a different perspective compared to previous methods, and thus it can effectively leverage techniques from existing methods to further enhance the performance of selective classification, which is empirically verified by combining with SAT~\cite{NIPS:SAT} and EM~\cite{Entropy+SR}.



\section{Related Work}


\subsection{Selective Classification}\label{sec-sc related}
Selective classification, also known as confidence-based classification, or classification with reject option~\cite{fitst-work}, allows a model to make predictions only when it is sufficiently confident, which has been extensively studied across multiple domains in machine learning, such as support vector machines~\cite{svm2008nips}, boosting~\cite{boostingnips2016}, nearest neighbours~\cite{knn1970}, online learning~\cite{onlinelearningpmlr}, and human assisted learning~\cite{liudx}.

With the widespread application of deep learning, the concept of selective classification for DNNs has been receiving increasing attention, especially in situations where incorrect predictions may lead to serious consequences.~\citet{GeifmanE17:coverage} proposed a method for converting trained DNNs into selective classifiers by employing two confidence functions, SR (defined as the maximal logit in the softmax layer) and MC-dropout (defined as the negative variance of aggregated predictive probabilities). 

Another type of selective classification method for DNNs is to modify the classification layer and train an additional selection head (or abstention logit). SN~\cite{SelectiveNet} is a three-headed network that includes prediction, selection, and auxiliary head, where the selection head is optimized to estimate the model's confidence in prediction for a given target coverage. DG~\cite{DeepGambler} expands the original $m$-class problem to a $(m+1)$-class problem, and uses the extra class to estimate the confidence of the model in abstention. Similarly, SAT~\cite{NIPS:SAT,SAT} also focuses its selection mechanism on the extra class and introduces a soft label-based training mechanism to guide the model in selecting which samples to abstain from predicting. However, recent experimental findings by~\cite{Entropy+SR} have shown that the methods utilizing their explicit selection heads as the confidence function are actually sub-optimal, and suggest using SR instead. 

While we have been focusing on selective classification, there are two important related topics, model calibration~\cite{calibration} and Human-AI collaboration system~\cite{aucoc}. Both of them focus on the confidence of the model. Model calibration adjusts the overall confidence level of the model to align its confidence with uncertainty, which can be divided into two categories: In-process and post-hoc methods. The in-process methods involve specifically designed loss functions to optimize calibration objectives, such as Soft AvUC/ECE loss~\cite{softavucece}, and MMCE loss~\cite{mmce}. The post-hoc methods globally adjust the confidence of the model after training, such as temperature scaling~\cite{calibration}, which, however, often do not change the ranking of confidence among samples, and thus cannot be directly used for selective classification. The definition of Human-AI collaboration~\cite{aucoc} is similar to that selective classification, which uses model confidence to determine which samples are delegated to human experts. The main difference is that the goal of Human-AI collaboration is more global, that is, optimizing AUCOC (Area Under Confidence Operating Characteristics), and a loss function was proposed to directly improve the AUCOC.  

\subsection{Contrastive Learning}
Instead of modifying the classification layer, we focus on optimizing feature representation for selective classification, which is facilitated by contrastive learning. Here, we introduce some related works on contrastive learning in both unsupervised and supervised domains.

Contrastive learning is a learning paradigm that maximizes the similarity between related samples and minimizes the similarity between unrelated samples, and has been commonly used for unsupervised representation learning.~\citet{infoNce} introduced a widely used form of contrastive loss function known as InfoNCE, which encourages the model to learn useful features by comparing each positive sample with multiple negative samples.~\citet{moco} proposed Momentum Contrast (MoCo), which uses a dynamic dictionary and a momentum encoder to solve the problem of insufficient diversity of negative samples caused by limited batch sizes in end-to-end training methods~\cite{infoNce,e2e}, as well as the problem of inconsistent features caused by slow update of features in memory bank methods~\cite{memorybank}. Recently, contrastive learning has also been introduced to supervised learning, where the label information is utilized to guide the division of positive and negative samples, aiming to obtain better feature representation.~\citet{supcon} proposed using samples with the same label as positives and those with different labels as negatives, and extended the \mbox{InfoNCE} loss to scenarios with multiple positives per anchor. 


While contrastive learning has been widely acknowledged as an effective approach for learning feature representations, its application in selective classification remains less explored. In this work, we leverage the strengths of contrastive learning to improve the feature representation of the model, enabling the model to better distinguish between correctly classified and incorrectly classified samples.

\section{Selective Classification Problem}

Let $\mathcal{X}$ and $\mathcal{Y}$ denote the feature space and the label space, respectively. Let $\mathcal{D}$ be an unknown data distribution over $\mathcal{X} \times \mathcal{Y}$. Let $\mathcal{F}$ and $\mathcal{G}$ denote two families of functions mapping $\mathcal{X}$ to ${[0,1]}^k$ and $[0,1]$, respectively. Our goal is to learn a selective classification model $(f,g)\in \mathcal{F} \times \mathcal{G}$:
\begin{align}\label{fg}
(f, g)(\bm{x}, y)= \begin{cases}f(\bm{x}) & \text { if } g(\bm{x})\geq h; \\ \text { Abstain } & \text { if } g(\bm{x})<h.\end{cases}
\end{align}

Here, $f: \mathcal{X} \rightarrow {[0,1]}^k$ represents a conventional classifier that outputs a probability vector for $k$ classes, with the predictive class $\hat{y}$ determined by $\hat{y}=  \arg\max_{j} f_j(\bm{x})$, and $g: \mathcal{X} \rightarrow [0, 1]$ is a selective function that estimates the confidence of $f(\bm{x})$ (also known as the confidence function), serving as a binary qualifier for $f$. That is, the model only predicts when $g(\bm{x})$ exceeds a predetermined threshold $h$.

Evaluating the performance of a selective classifier often involves two metrics: \emph{coverage} and \emph{selective risk}~\cite{GeifmanE17:coverage}. Coverage relies only on the selective function $g$, which is defined as:
\begin{align*}
\phi(g) \triangleq \mathbb{E}_{(\bm{x}, y) \sim\mathcal{D}}\; I[g(\bm{x}) \geq h],
\end{align*}
where the indicator function $I[\cdot]$ is $1$ if the inner expression is true and $0$ otherwise. Selective risk is defined as:
\begin{equation}
\label{selective risk}
R(f, g) \triangleq \frac{\mathbb{E}_{(\bm{x}, y) \sim\mathcal{D}}\{L[f(\bm{x}), y] \cdot I[g(\bm{\bm{x}}) \geq h]\}}{\phi(g)},
\end{equation}
where $L$ is typically the $0/1$ loss for classification. Thus, coverage $\phi(g)$ measures the ratio of instances that are classified by the model, and selective risk $R(f, g)$ measures the loss of the model when making predictions. In this paper, we follow the common modeling of the selective classification problem~\cite{SelectiveNet,NIPS:SAT,Entropy+SR}, i.e., to minimize the selective risk within a given target coverage $c_\text{target}$:   
\begin{equation}
\label{constraint}
\min _{} R\left(f, g\right)  \;\;
\text { s.t. }\; \phi\left(g\right) \geq c_\text{target} .
\end{equation}

\section{Theoretical Analysis}






In this section, we analyze the generalization performance of a DNN-based selective model for selective classification.
For the conventional classifier $f$ of a selective model $(f,g)$, we denote its feature embedding layer as $c$, and the final classification layer as $l$, i.e., $f = l \circ c: \mathcal{X} \rightarrow {[0, 1]^k}$. For a sample $\bm{x}$ with its corresponding label $y$, we denote the output of the feature embedding layer $c$ as $c(\bm{x})$, i.e., the non-normalized feature embedding of $\bm{x}$. 

For analytical convenience, we add the coverage constraint in Eq.~\eqref{constraint} to the objective function (i.e., selective risk in Eq.~\eqref{selective risk}) as a penalty term, yielding the following selective classification loss:
\begin{align*}
    L_{0}(f, g, \bm{x}, y) \!= L[f(\bm{x}), y] \!\cdot \!I[g(\bm{x})\geq h]\!+ \! \lambda \! \cdot  \! I[g(\bm{x}) < h], 
\end{align*}
where we use the $0/1$ loss $L[f(\bm{x}), y]\!=\!I[\arg\max_j f_j(\bm{x}) \!\neq\! y]$, and $\lambda> 0$ is the penalty coefficient which regulates the trade-off between minimizing selective risk and achieving high coverage rates to satisfy the constraint. Thus, the learning problem requires utilizing a set of labeled samples $S=\left\{\left(\bm{x}_1, y_1\right), \ldots,\left(\bm{x}_m, y_m\right)\right\}$, which are assumed to be independently and identically distributed and drawn from the unknown data distribution $\mathcal{D}^m$, to determine a pair $(f, g) \in \mathcal{F} \times \mathcal{G}$ that achieves a small expected selective classification loss $\mathbb{E}_{(\bm{x}, y) \sim \mathcal{D}}[L_{0}(f, g, \bm{x}, y)]$.



Margin loss is commonly used to analyze the generalization error of models~\cite{foundations,lyu2019refined,lyu2022improving,mdep,dep}. Here, we employ the margin loss associated with Max Hinge~\cite{margintheory} 
\begin{align*}
L_{\mathrm{MH}}^{\rho, \rho^{\prime}}(f, g, \bm{x}, y)=\max\! \Big\{&\max \{1+\frac{\alpha}{2}(\frac{g(\bm{x})}{\rho^{\prime}} \!-\!\frac{\gamma(\bm{x})}{\rho}), 0\},\\ 
 &\max \{\lambda (1-\beta \frac{g(\bm{x})}{\rho^{\prime}}), 0\}\Big\},
\end{align*}
as a surrogate loss function for the selective classification loss $L_{0}(f, g, \bm{x}, y)$, where $\rho, \rho^{\prime}$ are two parameters associated with the minimum margins of $f$ and $g$, respectively, $\alpha,\beta>0$, and $\gamma(\bm{x})$ is the margin of sample $\bm{x}$ on the classification layer of $f$, i.e., $\gamma(\bm{x}) \triangleq f_y(\bm{x}) - \max_{j\neq y}f_j(\bm{x})$. The margin loss $L_{\mathrm{MH}}^{\rho, \rho^{\prime}}(f, g, \bm{x}, y)$ is actually a convex upper bound on $L_{0}(f, g, \bm{x}, y)$. The first term $\max \{1+\frac{\alpha}{2}(\frac{g(\bm{x})}{\rho^{\prime}} -\frac{\gamma(\bm{x})}{\rho}), 0\}$ of $L_{\mathrm{MH}}^{\rho, \rho^{\prime}}$ indicates that the samples chosen for classification but classified incorrectly should be either correctly classified with a $\rho$-margin or modified in confidence to be rejected with a $\rho^{\prime}$-margin (i.e., emphasizing that the model will not predict incorrectly when choosing to make predictions). The second term $\max \{\lambda (1-\beta \frac{g(\bm{x})}{\rho^{\prime}}), 0\}$ indicates that every sample subjected to rejection should have its confidence adjusted to be selected for classification with a $\rho^{\prime}$-margin (i.e., emphasizing the model's selection of as many samples as possible to predict).

For a feed-forward DNN $f = l \circ c$ with ReLU activation function and its associated selective function $g$, we prove in Theorem~\ref{thm:bound} that the selective classification generalization error of $(f,g)$ can be bounded by the empirical margin loss $\mathbb{E}_{(\bm{x}, y) \sim S}[L_{\mathrm{MH}}^{\rho, \rho^{\prime}}(f, g, \bm{x}, y)]$ at its classification layer and another term positively related with $\operatorname{Var}_{\mathrm{intra}}[c(\bm{x})]$. Note that $\operatorname{Var}_{\mathrm{intra}}[c(\bm{x})]=\operatorname{tr}[\sum^k_{i =1} \operatorname{Cov}[c^i(\bm{x})] / k]$ denotes the intra-class variance, where $\operatorname{Cov}[c^i(\bm{x})]$ denotes the covariance matrix of the feature embeddings of all samples with label $i$, $k$ is the number of classes, and $\operatorname{tr}$ denotes the trace of a matrix. That is, $\operatorname{Var}_{\mathrm{intra}}[c(\bm{x})]$ denotes the variance of feature representations for samples within the same class. 
\begin{theorem}
\label{thm:bound}
$\forall \rho, \rho^{\prime}, \alpha,\beta,\lambda>0$, and $\forall \delta>0$, with probability at least $1-\delta$ over a training set of size $m$, we have:
\begin{align*}
&\mathbb{E}_{(\bm{x}, y) \sim \mathcal{D}}[L_{0}(f, g, \bm{x}, y)] \leq \mathbb{E}_{(\bm{x}, y) \sim S}\left[L_{\mathrm{MH}}^{\rho, \rho^{\prime}}(f, g, \bm{x}, y)\right]\\
&\quad +4 \sqrt{\frac{\|l\|_2^2 \operatorname{Var}_{\mathrm{intra}}[c(\bm{x})]+4 \tilde{\rho}^2+\tilde{\rho}^2\|l\|_2^2 \ln \frac{6 m}{\delta}}{\tilde{\rho}^2 m\|l\|_2^2}},
\end{align*}
where $\|l\|_2$ denotes the L2-norm of the classification layer $l$'s parameters, and $\Tilde{\rho}=\min \{\rho/(4 \alpha), \rho^{\prime}/(4 \beta \lambda+2 \alpha)\}$. 
\end{theorem}





\mbox{Theorem~\ref{thm:bound}} discloses that the generalization performance of a selective classification model is associated not only with the empirical margin loss at the classification layer but also with the feature representation at the feature layer. Specifically, a smaller intra-class variance of feature representation will enhance the generalization performance of selective classification. The proof is accomplished by bounding the Kullback–Leibler divergence term in a PAC-Bayesian lemma with the variance of perturbation, which is related to the intra-class variance $\operatorname{Var}_{\mathrm{intra}}[c(\vx)]$ of feature representation and the margin parameter $\tilde{\rho}$, and the proof details are provided in Appendix~\ref{proofs-sec} due to space limitation.

\section{CCL-SC Method}\label{method}

For selective classification, previous works focus on modifying the classification layer of a model to enable the model to better estimate its prediction confidence~\cite{SelectiveNet,DeepGambler,NIPS:SAT,Entropy+SR}. Inspired by Theorem~\ref{thm:bound}, we improve the model's selective classification performance from a new perspective, that is, we optimize the feature representation of the model to aggregate the feature representations of samples in the same category, which aligns naturally with the paradigm of contrastive learning. For the loss function of the classification layer, we utilize the cross-entropy loss, which is a smooth relaxation of the margin loss and is easier to optimize for DNNs~\cite{smooth}. SR~\cite{SR,GeifmanE17:coverage}, i.e., the maximum predictive class score $\max_j f_{j}(\bm{x})$, is used as the selective function $g$ in our method.




Contrastive learning has been used in supervised learning~\cite{supcon}, which simply defines positive and negative samples as those with the same label and different labels, respectively, and optimizes the feature representation by pulling positive samples closer than negative samples. However, previous contrastive learning methods do not consider the correctness and confidence of the prediction for the samples~\cite{memorybank,moco,simclr,supcon,CDS}, which cannot directly meet the requirement of selective classification: Correctly classified samples should be separated from misclassified samples, and the model's confidence function should reflect the reliability of its classifications, i.e., samples with higher confidence are more likely to be classified correctly. To address this issue, we redefine the positive and negative samples according to the predicted results of the current model: a sample is positive/negative if the prediction matches the anchor’s label and is correct/incorrect. Then, we design a new contrastive loss function to separate the feature representations of correctly classified and misclassified samples by making the anchor's features more similar to its positive samples and less similar to its negative samples; and to pay more attention to samples with higher confidence during training by weighting the loss with the model's SR.

\begin{figure*}[htbp]  \centering  \includegraphics[width=0.8\textwidth]{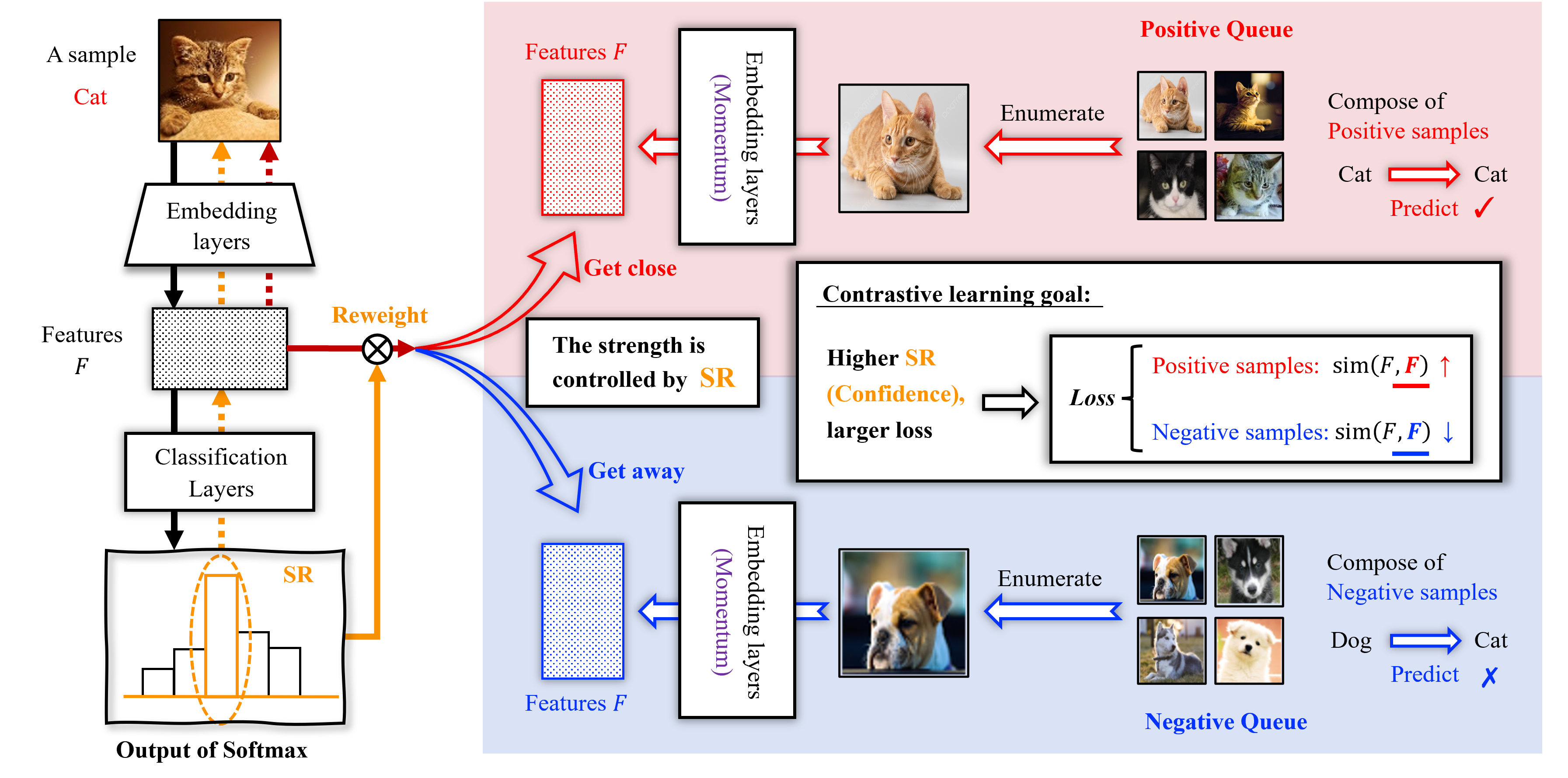}\vspace{-0em}
\caption{Illustration of the proposed CCL-SC method. The right part outlines our definition of positive/negative samples: a sample is positive/negative if the prediction matches the anchor's label and is correct/incorrect. Two independent queues store positive and negative samples, respectively. The middle part displays the characteristic of the proposed CSC loss: prompting the model to separate correctly classified and misclassified samples at the feature level and focus on samples with high prediction confidence. The black arrow on the left represents forward calculation, while the yellow and red ones represent backpropagation of the cross-entropy and CSC loss, respectively.}   \label{FIG:method} \vspace{-0em}\end{figure*}

Figure~\ref{FIG:method} briefly illustrates our proposed method CCL-SC. In the following, we will provide detailed introduction to its key components, i.e.,
\begin{itemize}
    \item How to define and construct positive/negative samples?
    \item How to design a loss function to incentivize the model to learn features conducive to selective classification, and to make it sensitive to the model's SR?
    \item How to use the proposed loss to train the model?
    \\
\end{itemize}

\subsection{Constructing Positive and Negative Samples}\label{sec-samples}

To enhance the model's ability to discern the correctness of predictions, we define positive and negative samples based on whether the model predicts the sample correctly. A positive sample $\bm{x}_p$ for the anchor $\bm{x}$ with label $y$ is defined as a sample that is predicted to belong to the class $y$ and is correctly classified. A negative sample $\bm{x}_n$ is defined as a sample that is incorrectly predicted to belong to the class $y$. In other words, although the true label of the negative sample is different from $y$, the model incorrectly classifies it into the same class. The right part of Figure~\ref{FIG:method} illustrates the definition of positive and negative samples.

Due to the dependency on the model predictions, positive and negative samples have to be sampled from batches. However, sampling solely from the current mini-batch is often insufficient, particularly when dealing with a large number of classes, such as the ImageNet dataset with 1000 classes. To address this issue, we adopt the approach used in MoCo~\cite{moco}, introducing queues as dictionaries to reuse samples from different batches, while utilizing a momentum encoder (denoted as ${f_{\bm{\theta}}}_{\mathrm{m}}$) to generate sample features. The parameters $\bm{\theta}_\text{m}$ of the momentum encoder are updated based on the original encoder parameters $\bm{\theta}$:\vspace{-0.2em}
\begin{equation}
\label{update}
\bm{\theta}_{\mathrm{m}} \leftarrow q\cdot  \bm{\theta}_{\mathrm{m}}+(1-q)\cdot  \bm{\theta},\vspace{-0.2em}
\end{equation}
where $q \in [0,1)$ is the momentum coefficient to control the magnitude of updates. But unlike MoCo, our method maintains two separate queues $P$ and $Q$, to store positive and negative samples, respectively. Each element in the queue is a tuple composed of the normalized feature and the predicted class of a sample. We denote $P(y)\subseteq P$ and $N(y)\subseteq N$ as the positive samples and negative samples for an anchor with label $y$. For simplicity, we use queues of the same size $s$ to store positive and negative samples.\vspace{-0.3em}

\subsection{Confidence-aware Supervised Contrastive Loss}

Let $\bm{z}_i = c(\bm{x}_i) / \lVert c(\bm{x}_i) \rVert$ denote the normalized feature embedding of sample $\bm{x}_i$ through the embedding layer $c$. Then, $\bm{z}_i \cdot \bm{z}_j$ is just the cosine similarity $c(\bm{x}_i) \cdot c(\bm{x}_j)/(\|c(\bm{x}_i)\|\cdot\|c(\bm{x}_j)\|)$ between the feature representations of samples $\bm{x}_i$ and $\bm{x}_j$. We introduce a new loss function called Confidence-aware Supervised Contrastive (CSC) loss, by combining a contrastive learning loss in the form of infoNCE~\cite{infoNce} with the predictive confidence SR of the model. Given an anchor $\bm{x}$ with label $y$, its CSC loss is defined as \vspace{-0.1em}
\begin{equation}
\begin{aligned}
L_{\text {CSC}} =  \frac{ \max_j f_{j}(\bm{x})}{-|P(y)|} \! \! \sum_{\bm{x}_p \in P(y)}\! \! \log \frac{\exp \left(\bm{z} \cdot \bm{z}_p / \tau\right)}{ \sum_{\bm{x}_a \in A(y)} \exp \left(\bm{z} \cdot \bm{z}_a / \tau\right)},\nonumber
\end{aligned}\vspace{-0.1em}
\end{equation}
where $\max_j \!f_{j}(\bm{x})$ is SR, serving as a weight coefficient to dynamically adjust the magnitude of $L_{\text {CSC }}$, $A(y)\! \!= \!\!N(y) \! \cup \! \{\bm{x}_p\}$ denotes the set of all negative samples of the anchor $\bm{x}$ and the current positive sample $\bm{x}_p$, and $\tau$ is a temperature hyper-parameter controlling the emphasis of $L_{\text {CSC }}$ on difficult samples. We set $\tau$ to a commonly used value of 0.1. $L_{\text {CSC }}$ compels the anchor to be closer to the positive samples and farther from the negative samples, thereby achieving the goal of distinguishing between correctly classified and misclassified samples; and focusing on the samples with high confidence by combining SR.
 


\textbf{Enhanced feature learning for selective classification.} Unlike previous contrastive learning methods, the CSC loss utilizes both the label and prediction information of samples, providing strong supervision for each anchor. Through the contrastive learning of multiple positive and negative samples, the CSC loss encourages the model to learn features that can better distinguish between correct and incorrect predictions, enabling the model to learn more robust embedding spaces. Note that to ensure a fair comparison with previous selective classification methods, we do not acquire positive samples through additional data augmentation techniques, which was common in contrastive learning.

\textbf{Mining samples with high confidence but poor features.} 
The gradient of $L_{\text{CSC}}$ with respect to the normalized feature embedding $\bm{z}$ of the anchor $\bm{x}$ can be calculated as
$$
\frac{\max_j \!f_{j}(\bm{x})}{-\tau}\!\Bigg(\sum_{\bm{x}_p \in P(y)}\!\!\Big(\frac{1}{|P(y)|}\!-\!X_{\bm{z},\bm{z}_p}\!\Big) \bm{z}_p\!-\!\!\!\!\sum_{\bm{x}_n \in N(y)}\!\!\!\! \!X_{\bm{z},\bm{z}_n} \bm{z}_n\! \!\Bigg),
$$
where $X_{\bm{z}, \bm{z}_j}\!\!=\!\exp(\bm{z} \cdot \bm{z}_j / \tau)/\!\sum_{\bm{x}_a \in N(y) \cup\{\bm{x}_j\}} \!\exp (\bm{z} \cdot \bm{z}_a / \tau)$. Thus, the higher the value of the confidence SR (i.e., $\max_j \!f_{j}(\bm{x})$) and the poorer the feature representation (i.e., dissimilar from positive samples and similar to negative samples), the larger the gradient scale of the CSC loss for $\bm{z}$. By modulating the loss with the model's SR, the CSC loss directly prioritizes learning from instances where the model is more certain, since the loss of the samples with low prediction confidence is small. Thus, the features of samples with low prediction confidence will be distinguished from those of high-confidence samples, improving the discriminability of samples with different confidence levels. Similar to the supervised contrastive loss in~\cite{supcon}, the CSC loss has the ability to mine difficult samples, that is, the model will pay more attention to samples with high prediction confidence but poor feature representation. 


\subsection{Training with CSC Loss}

We now introduce the training method with the proposed CSC loss. Different from the conventional two-stage training manner ``pre-train then finetune'' used in typical contrastive learning methods~\cite{moco,simclr,supcon}, we employ a one-stage manner to optimize the classification layers and convolutional layers of the model together, because the CSC loss involves the predictive information of the model. 

The model $f$ is initially trained using the cross-entropy loss $L_\text{CE}$ for $E_s$ epochs. When $e \geq E_s$ and meanwhile both queues $P$ and $Q$ have been updated more than $s$ tuples, the training gradient consists of two parts: the gradient of the CSC loss $L_\text{CSC}$ returned from the last feature embedding layer $c$, and the gradient of the cross-entropy loss $L_\text{CE}$ returned from the classification layer $l$. We use a weight coefficient $w$ to balance these two loss items. The parameters $\bm{\theta}$ of the model $f$ will be updated using an optimizer based on the combined loss. As introduced in Section~\ref{sec-samples}, when $e \geq E_s$, a momentum encoder ${f_{\bm{\theta}}}_{\mathrm{m}}$ is used to generate positive and negative samples. When the epoch $e$ equals $E_s$, the parameters $\bm{\theta}_\mathrm{m}$ of the momentum encoder ${f_{\bm{\theta}}}_{\mathrm{m}}$ are initialized to the parameters $\bm{\theta}$ of the current model $f$, which ensures that the momentum encoder has a favorable initial accuracy. After that, the parameters $\bm{\theta}_\mathrm{m}$ of the momentum encoder will be updated according to Eq.~\eqref{update} at each training step. The samples generated by the momentum encoder will be used to update the queues $P$ and $Q$, which store positive and negative samples, respectively. Note that ${f_{\bm{\theta}}}_{\mathrm{m}}$ will not be optimized by the optimizer. The pseudo-code of the training method is shown in Algorithm~\ref{alg:algorithm} in Appendix~\ref{train-sec}. 

\section{Experiments}

In this section, we will give the experimental settings and results. Due to space limitation, some details are shown in Appendix~\ref{setting sec} to~\ref{improve-sec}. The codes are provided in \url{https://github.com/lamda-bbo/CCL-SC}.

\textbf{Datasets} We conduct experiments on four commonly used datasets, i.e., CelebA~\cite{celeba}, CIFAR-10/CIFAR-100~\cite{cifar}, and ImageNet~\cite{imagenet}. CelebA is a large-scale face attributes dataset, consisting of over 200,000 celebrity images, and the challenging label `attractive' is used as the target for binary classification. CIFAR-10 and CIFAR-100 are two datasets containing images across 10 and 100 categories, with 5,000 and 500 images per category, respectively. ImageNet contains 1000 categories of images, with 1300 images per category. The experiments on CIFAR-10, CIFAR-100, and CelebA are run with 5 seeds, and those on ImageNet are run with 3 seeds.

\textbf{Baselines} We compare our method against SOTA selective classification methods, including SAT~\cite{NIPS:SAT,SAT}, SAT with Entropy-Minimization regularization (SAT+EM)~\cite{Entropy+SR}, and DG~\cite{DeepGambler}. Based on~\cite{Entropy+SR}, we include the results of using SR as the confidence function for these methods as well. We also compare with a common baseline denoted as SR~\cite{SR,GeifmanE17:coverage}, which is a vanilla classifier trained via the cross-entropy loss and uses SR as the confidence function. We do not compare with SN~\cite{SelectiveNet} since previous methods such as SAT+EM have already demonstrated their superiority over SN. Additionally, SN requires retraining for different coverage rates, resulting in significantly higher training costs compared to other methods.
\begin{table*}[t!]
\caption{Selective risk (\%) on CIFAR-10 for various coverage rates (\%). The mean and standard deviation are calculated over 5 trials. The best entries are marked in bold. The symbol `$\bullet$'/`$\circ$' indicates that CCL-SC is significantly better/worse than the corresponding method, according to the Wilcoxon rank-sum test with significance level 0.05.}\vskip 0.05in
\small
\fontsize{9}{10}\selectfont
\centering
    \begin{tabular}{ccccccccc}
    \toprule 
    
Coverage & CCL-SC & SAT+EM+SR & SAT+EM & SAT+SR & SAT   & DG+SR & DG    & SR \\
\midrule
100   & \textbf{5.97±0.11} & 6.14±0.07 & 6.14±0.07 & 6.16±0.13• & 6.16±0.13• & 6.34±0.16• & 6.34±0.16• & 6.25±0.07\  \\
    99    & \textbf{5.32±0.05} & 5.61±0.06• & 5.58±0.06• & 5.63±0.09• & 5.63±0.11• & 5.77±0.16• & 5.75±0.18• & 5.69±0.07• \\
    98    & \textbf{4.87±0.04} & 5.11±0.07• & 5.14±0.03• & 5.16±0.07• & 5.14±0.13• & 5.27±0.19• & 5.23±0.19• & 5.16±0.06 \\
    97    & \textbf{4.41±0.07} & 4.66±0.09• & 4.69±0.06• & 4.67±0.05• & 4.74±0.11• & 4.82±0.16• & 4.70±0.20 & 4.67±0.04• \\
    95    & \textbf{3.56±0.06} & 3.85±0.09• & 3.93±0.06• & 3.87±0.08• & 3.97±0.11• & 3.98±0.13• & 3.83±0.13• & 3.81±0.07• \\
    90    & \textbf{2.01±0.07} & 2.20±0.07• & 2.34±0.09• & 2.26±0.12• & 2.35±0.16• & 2.37±0.13• & 2.26±0.10• & 2.19±0.12• \\
    85    & \textbf{1.10±0.06} & 1.23±0.07• & 1.31±0.07• & 1.27±0.09• & 1.36±0.10• & 1.39±0.09• & 1.32±0.11• & 1.25±0.13• \\
    80    & 0.69±0.08 & \textbf{0.67±0.07} & 0.71±0.07 & 0.71±0.09 & 0.74±0.09 & 0.86±0.08• & 0.72±0.10 & 0.68±0.08 \\\bottomrule
    \end{tabular}\vspace{-0.5em}
\label{result-CIFAR-10}

\end{table*}

\textbf{Hyper-parameters} For each dataset, we utilize 20\% of the training set as the validation set to tune hyper-parameters. We test the momentum coefficient $q\in \{0.9, 0.999, 0.999\}$, and the weight coefficient $w\in \{0.1, 0.5, 1.0\}$. For the queue size $s$, we set it based on the number of classes in the dataset: for datasets with fewer classes such as CelebA and CIFAR-10, $s=300$; for datasets with more classes such as CIFAR-100 and ImageNet, we set $s=3000$ and $s=10000$, respectively. We train the model on the entire training set to evaluate performance. Detailed hyperparameter settings for each method on each dataset are provided in Appendix~\ref{param sec}.

\textbf{Networks and Training} Following prior work, for CIFAR-10 and CIFAR-100, we employ VGG16~\cite{vgg} as the backbones of selective classifiers. ResNet34 and ResNet18~\cite{resnet} are utilized for ImageNet and CelebA, respectively. For the same dataset, all compared methods utilize the same data augmentation and training parameters to ensure a fair comparison. The detailed settings of training parameters and data augmentation are provided in Appendix~\ref{network sec} and~\ref{Data Augmentation sec}, respectively.

\subsection{Comparison with State-of-the-art Methods}

Table~\ref{result-CIFAR-10} shows the results on CIFAR-10, which has a low classification difficulty but is most widely used in selective classification. It can be observed that the selective risk of all methods on CIFAR-10 is below 1\% at coverage 80\%. We do not include the results for lower coverage as the selective risk approaches almost zero thereafter. Our method CCL-SC achieves the lowest selective risk when coverage is at least 85\%. However, we can also see that the discrimination of different methods on CIFAR-10 is relatively low, which is consistent with previous works~\cite{NIPS:SAT,SAT,Entropy+SR}, mainly due to the insufficient number of misclassified samples. As a result, we focus on the other three datasets that are not saturated with accuracy.

Table~\ref{result-celeba&cifar100} shows the selective risk of different methods at various coverage on CelebA and CIFAR-100. For CelebA, CCL-SC performs the best across all degrees of coverage. Compared to any other method, CCL-SC performs significantly better on at least $7/13$ coverage rates, according to the Wilcoxon rank-sum test~\cite{rank-sum} with significance level 0.05. For CIFAR-100, CCL-SC achieves the lowest selective risk when coverage is at least 50\% or equal to 10\%, and is only worse than SAT-related methods in the range of 20\% to 40\%. According to the significance test, CCL-SC is significantly better than SAT-related methods when the coverage is at least 60\%, and only significantly worse than SAT+SR at 30\% coverage. 

We can also observe from Table~\ref{result-celeba&cifar100} that those methods (i.e., SAT+EM, SAT, DG) based on the additional selection head will generally become better if using SR as the confidence function, as observed in~\cite{Entropy+SR}. Furthermore, the improvement is larger on CIFAR-100 than on CelebA with only two classes, which could be attributed to the increased difficulty of learning additional logits as the number of classes increases. Note that the relative rankings of SAT+EM, SAT, and DG in our experiments are also consistent with the overall rankings reported in previous works~\cite{NIPS:SAT,SAT,Entropy+SR}.

\begin{table*}[t!]
\caption{Selective risk (\%) on CelebA and CIFAR-100 for various coverage rates (\%). The mean and standard deviation are calculated over 5 trials. The best entries are marked in bold. The symbol `$\bullet$'/`$\circ$' indicates that CCL-SC is significantly better/worse than the corresponding method, according to the Wilcoxon rank-sum test with significance level 0.05. The w/t/l denotes the number of cases where the selective risk of CCL-SC is significantly lower, almost equivalent, or significantly higher, compared to the corresponding method.}\vskip 0.05in
\small
\fontsize{8}{10}\selectfont
\centering
    \begin{tabular}{ccccccccc}
    \toprule 
    \multicolumn{9}{c}{CelebA} \\\midrule
Coverage & CCL-SC & SAT+EM+SR & SAT+EM & SAT+SR & SAT   & DG+SR & DG    & SR \\
\midrule
100   & \textbf{18.71±0.16} & 19.04±0.30 & 19.04±0.30 & 19.20±0.37• & 19.20±0.37• & 19.26±0.23• & 19.26±0.23• & 19.30±0.13• \\
    95    & \textbf{17.00±0.09} & 17.55±0.33• & 17.74±0.31• & 17.65±0.33• & 17.82±0.32• & 17.82±0.26• & 18.06±0.29• & 17.80±0.12• \\
    90    & \textbf{15.47±0.14} & 16.06±0.36• & 16.24±0.31• & 16.17±0.31• & 16.38±0.30• & 16.27±0.32• & 16.87±0.28• & 16.30±0.12• \\
    85    & \textbf{14.03±0.17} & 14.56±0.36• & 14.72±0.28• & 14.65±0.27• & 14.92±0.21• & 14.81±0.27• & 15.61±0.27• & 14.81±0.11• \\
    80    & \textbf{12.51±0.21} & 13.12±0.39• & 13.21±0.29• & 13.17±0.29• & 13.38±0.21• & 13.43±0.34• & 14.36±0.30• & 13.37±0.13• \\
    75    & \textbf{11.05±0.20} & 11.69±0.37• & 11.72±0.36• & 11.78±0.26• & 11.90±0.19• & 12.11±0.32• & 13.01±0.33• & 11.98±0.15• \\
    70    & \textbf{9.73±0.14} & 10.33±0.32• & 10.35±0.28• & 10.41±0.27• & 10.49±0.23• & 10.81±0.37• & 11.73±0.40• & 10.62±0.14• \\
    60    & \textbf{7.16±0.09} & \ \ 7.84±0.43•& \ \ 7.73±0.37• & \ \ 8.00±0.35• & \ \ 7.94±0.28• & \ \ 8.28±0.43• & \ \ 8.99±0.43• & \ \ 8.06±0.29• \\
    50    & \textbf{4.93±0.16} & \ \ 5.51±0.44• & 5.41±0.42 & \ \ 5.71±0.25• & \ \ 5.68±0.18• & \ \ 6.17±0.58• & \ \ 6.27±0.37• & \ \ 5.92±0.23• \\
    40    & \textbf{3.09±0.14} & 3.66±0.54 & 3.50±0.49 & \ \ 3.82±0.22• & \ \ 3.77±0.17• & \ \ 4.35±0.67• & \ \ 3.86±0.28• & \ \ 4.03±0.24• \\
    30    & \textbf{1.87±0.15} & 2.16±0.45 & 2.06±0.35 & \ \ 2.33±0.16• & 2.15±0.21 & 2.81±0.68 & 2.16±0.26 & \ \ 2.49±0.24• \\
    20    & \textbf{0.92±0.12} & 1.10±0.25 & 1.02±0.21 & \ \ 1.21±0.14• & 1.12±0.13 & \ \ 1.65±0.37• & 1.10±0.12 & \ \ 1.45±0.12• \\
    10    & \textbf{0.25±0.08} & 0.41±0.12 & 0.37±0.12 & \ \ 0.54±0.12• & \ \ 0.49±0.06• & \ \ 0.80±0.21• & 0.41±0.14 & \ \ 0.50±0.12• \\
    \midrule
    w/t/l & \multicolumn{1}{c}{/} & \multicolumn{1}{c}{8/5/0} & \multicolumn{1}{c}{7/6/0} & \multicolumn{1}{c}{14/0/0} & \multicolumn{1}{c}{12/2/0} & \multicolumn{1}{c}{13/1/0} & \multicolumn{1}{c}{11/3/0} & \multicolumn{1}{c}{14/0/0} \\
    \midrule
    Avg. Rank & \textbf{1.00 } & 2.54  & 2.77  & 4.46  & 5.00  & 6.85  & 6.62  & 6.15  \\
    \midrule
    \multicolumn{9}{c}{CIFAR-100} \\\midrule
    Coverage & CCL-SC & SAT+EM+SR & SAT+EM & SAT+SR & SAT   & DG+SR & DG    & SR \\
\midrule
100   & \textbf{26.55±0.26} & 26.96±0.14• & 26.96±0.14• & 26.98±0.16• & 26.98±0.16• & 27.12±0.30• & 27.12±0.30• & 27.19±0.33• \\
    95    & \textbf{23.54±0.15} & 24.14±0.12• & 24.16±0.10• & 24.17±0.18• & 24.25±0.14• & 24.28±0.26• & 24.35±0.38• & 24.37±0.29• \\
    90    & \textbf{20.97±0.20} & 21.52±0.22• & 21.56±0.10• & 21.59±0.15• & 21.68±0.16• & 21.72±0.29• & 21.84±0.40• & 21.86±0.31• \\
    85    & \textbf{18.57±0.20} & 19.08±0.21• & 19.14±0.18• & 19.09±0.22• & 19.18±0.19• & 19.35±0.16• & 19.43±0.32• & 19.43±0.40• \\
    80    & \textbf{16.07±0.15} & 16.71±0.18• & 16.73±0.23• & 16.64±0.19• & 16.84±0.23• & 16.89±0.19• & 17.20±0.37• & 17.09±0.45• \\
    75    & \textbf{13.60±0.19} & 14.30±0.19• & 14.53±0.28• & 14.21±0.23• & 14.49±0.22• & 14.44±0.29• & 14.97±0.48• & 14.58±0.39• \\
    70    & \textbf{11.23±0.16} & 11.94±0.21• & 12.07±0.20• & 11.83±0.18• & 12.11±0.20• & 12.05±0.42• & 12.81±0.64• & 12.28±0.29• \\
    60    & \textbf{6.83±0.15} & \ \ 7.51±0.16• & \ \ 7.83±0.07• & \ \ 7.54±0.09• & \ \ 7.79±0.22• & \ \ 7.75±0.61• & \ \ 8.91±0.75• & \ \ 7.71±0.33• \\
    50    & \textbf{3.95±0.22} & 4.08±0.15 & \ \ 4.30±0.15• & 4.10±0.24 & 4.32±0.19 & 4.40±0.56 & \ \ 5.48±0.72• & \ \ 4.36±0.17• \\
    40    & 2.29±0.33 & 2.12±0.15 & 2.37±0.20 & \textbf{2.00±0.04} & 2.38±0.10 & 2.45±0.50 & \ \ 3.16±0.49• & 2.18±0.13 \\
    30    & 1.26±0.17 & 1.05±0.11 & 1.37±0.16 & \ \ \textbf{0.96±0.10}$\circ$ & 1.21±0.10 & 1.68±0.59 & 2.07±0.57 & 1.29±0.11 \\
    20    & 0.71±0.12 & \textbf{0.54±0.15} & 0.77±0.08 & \textbf{0.54±0.14} & 0.67±0.12 & 1.14±0.50 & \ \ 1.65±0.49• & 0.78±0.11 \\
    10    & \textbf{0.36±0.08} & 0.48±0.12 & 0.58±0.21 & 0.42±0.16 & 0.48±0.23 & \ \ 0.74±0.29• & \ \ 1.28±0.30• & 0.58±0.21 \\
    \midrule
    w/t/l & /     & 8/5/0 & 9/4/0 & 8/4/1 & 8/5/0 & 9/4/0 & 12/0/0/ & 9/4/0 \\
    \midrule
    Avg. Rank & \textbf{1.69 } & 2.23  & 4.54  & 2.46  & 4.69  & 6.00  & 7.62  & 6.23  \\
    \bottomrule
    \end{tabular}\vspace{-0em}
\label{result-celeba&cifar100}\vspace{-0em}
\end{table*}



\begin{table}[t!]
\caption{Selective risk (\%) on ImageNet for various coverages~(\%). The mean and standard deviation are calculated over 3 trials. The best and runner-up entries are bolded and underlined, respectively.}\vskip 0.05in
\fontsize{8}{10}\selectfont
\centering
\begin{tabular}{ccccc}
\toprule 
Cov. & CCL-SC & SAT+EM+SR & SAT & CCL-SC*\\
    \midrule
100   & \textbf{26.26±0.10} & \underline{27.27 ± 0.05} & 27.41 ± 0.08 & 27.31±0.04 \\
    90    & \textbf{20.68±0.07} & \underline{21.57 ± 0.19} & 22.67 ± 0.24  & 21.71±0.03 \\
    80    & \textbf{15.76±0.07} & 16.83 ± 0.06 & 18.14 ± 0.28 & \underline{16.78±0.03} \\
    70    & \textbf{11.39±0.10} & 12.34 ± 0.11 & 13.88 ± 0.14 & \underline{12.25±0.04} \\
    60    & \textbf{7.55±0.09} & 8.45 ± 0.05 & 10.11 ± 0.15 & \underline{8.34±0.03} \\
    50    & \textbf{4.79±0.04} & 5.57 ± 0.17 & 6.82 ± 0.07 & \underline{5.33±0.06} \\
    40    & \textbf{2.95±0.04} & 3.77 ± 0.00 & 4.32 ± 0.33 & \underline{3.35±0.05} \\
    30    & \textbf{1.83±0.05} & 2.32 ± 0.15 & 2.68 ± 0.14 & \underline{2.03±0.04} \\
    20    & \textbf{1.22±0.05} & 1.35 ± 0.20 & 1.82 ± 0.13 & \underline{1.23±0.04} \\
    10    & 0.72±0.05 & \textbf{0.55 ± 0.05} & 1.27 ± 0.34 & \underline{0.68±0.07} \\ \bottomrule
    \end{tabular}\vspace{-0.5em}
\label{result-imagenet}\vspace{-1.69em}
\end{table}

Table~\ref{result-imagenet} shows the comparison with the two currently best-performing methods SAT+EM+SR and SAT on ImageNet. Note that their results are directly from~\cite{Entropy+SR}, and the comparison is fair as all training settings are same. We can observe that CCL-SC always performs the best, except at 10\% coverage. However, one may argue that the performance improvement is due to the accuracy improvement, because CCL-SC achieves 1\% accuracy improvement over other methods at full coverage. To mitigate this concern, we also load checkpoints when training is completed in just 100 epochs, and the results are shown in the CCL-SC* column of Table~\ref{result-imagenet}. Now our method achieves slightly lower accuracy than SAT+EM+SR at full coverage, but still outperforms both SAT+EM+SR and SAT when the coverage is between 20\% and 80\%.

\subsection{Comparison with Other Related Methods}
As described in Section~\ref{sec-sc related}, selective classification is closely related to model calibration and Human-AI collaboration. In this section, we introduce five methods from model calibration, including Focal loss~\cite{focal_loss}, Adaptive Focal loss~\cite{afocalloss}, Soft AvUC loss~\cite{softavucece}, Soft ECE loss~\cite{softavucece}, and MMCE loss~\cite{mmce}, as well as AUCOC loss~\cite{aucoc} from Human-AI collaboration into selective classification. Tabel~\ref{result-other-method} in Appeidx~\ref{ref-compare-other} shows that CCL-SC has the lowest selective risk at various coverage rates compared to these methods.

\subsection{Alignment between Proposed Theory and Method}
\mbox{Theorem~\ref{thm:bound}} discloses that a smaller intra-class variance of feature representation of a model will enhance its generalization performance of selective classification, and our method CCL-SC explicitly optimizes this aspect. Although previous experiments have confirmed the superiority of our method in the generalization performance of selective classification, we still need to answer through experiments whether our method has really obtained lower intra-class variance and tighter bounds in \mbox{Theorem~\ref{thm:bound}} compared to other selective classification methods. We show the changes of the intra-class variance and the bound in \mbox{Theorem~\ref{thm:bound}} of different methods during the training process on CIFAR-100 in Figure~\ref{fig:theorem-method} in Appendix~\ref{sec-theoty-method}. It can be observed that CCL-SC does have the lowest intra-class variance and the lowest bound. It is worth mentioning that the relative order of the intra-class variance and the bounds of different methods is consistent with their actual order of selective risk, and the methods with similar selective risk (such as SAT+EM and SAT) also have similar intra-class variance/bound, indicating the importance of intra-class variance for selective classification performance and the usefulness of our bound.

\subsection{Learned Feature Representation}
Because our method CCL-SC explicitly optimizes the feature layer of the model, we compare its learned feature representations with those of SR trained only using the cross-entropy loss on CIFAR-10 at coverage 95\%. The t-SNE~\cite{t-SNE} visualization shown in Figure~\ref{tSNE-cifar10} in Appendix~\ref{tsne-sec} clearly demonstrates that CCL-SC achieves more significant inter-class separation and intra-class aggregation in the feature space, which confirms that optimizing the feature layer contributes to performance improvement in selective classification.

\subsection{Ablation Studies and Hyper-parameter Sensitivity}\label{params}

Next, we conduct ablation studies and parameter sensitivity analysis on CIFAR-100 for the proposed method CCL-SC. 

\textbf{SR-weighted} We verify whether the SR-weighted manner in the proposed CSC loss $L_{\text{CSC}}$ really improves the performance of selective classification. Table~\ref{result-Ablation} in Appendix~\ref{ab-sec} shows that the original CCL-SC method using the SR-weighted CSC loss consistently achieves lower selective risk than that using the unweighted CSC loss across all coverage degrees. 

\textbf{The contrastive learning method of CCL-SC} We first conduct ablation experiments for the construction of negative samples. For convenience, we name the ablation method CCL-SC2. For the negative samples of the samples that are correctly classified as class $y$, CCL-SC2 contains not only the samples misclassified as class $y$ in the queue defined in CCL-SC, but also the samples from other classes in the queue. Table~\ref{result-ccl-sc2} in Appendix~\ref{ab-sec} shows that even with the addition of negative samples, the performance of CCL-SC2 will not be improved compared to the original CCL-SC. This implies that the improvement of the performance is likely to be from the negative samples we define. To confirm this conclusion, we conduct another ablation study in which we remove the negative samples defined in CCL-SC, and only use randomly sampled samples from other categories as negative samples. For simplicity, we name this ablation method CCL-SC3. Table~\ref{result-ccl-sc3} in Appendix~\ref{ab-sec} shows that CCL-SC3 has a significant performance decrease compared to CCL-SC, which demonstrates the effectiveness of our strategy to construct negative samples.

We also conduct ablation experiments for the whole contrastive method of CCL-SC. Specifically, we introduce the positive and negative sample definition method and loss function from \cite{supcon} into our CCL-SC method, while keeping the other components consistent. We name this ablation method CCL-SC+SupCon. The comparison results on CIFAR-100 are shown in Table~\ref{result-ccl-sc3} in Appendix~\ref{ab-sec}. It can be observed that the selective classification performance of the original CCL-SC is better than CCL-SC+SupCon.

\textbf{Hyper-parameter Sensitivity} We then analyze the influence of four important hyper-parameters: the momentum coefficient $q$, queue size $s$, weight coefficient $w$, and initial epochs $E_s$. Tables~\ref{result-m} to~\ref{result-Es} in Appendix~\ref{ab-sec} show that the performance of CCL-SC is generally not sensitive to their settings, i.e., CCL-SC can achieve good performance in a wide range of these hyper-parameters. Detailed results can be found in Appendix~\ref{ab-sec}.

\subsection{Combination of CCL-SC and Existing Methods}\label{improvement}

Finally, we are to verify another benefit of CCL-SC, i.e., it can be seamlessly integrated with existing methods that optimize the model at the classification layer, because CCL-SC operates on the feature representation of the model. Here, we combine CCL-SC with SAT~\cite{NIPS:SAT,SAT} and EM~\cite{Entropy+SR} methods. That is, the loss function at the classification layer is modified from the cross-entropy loss to the loss of SAT with EM regularization. Table~\ref{result-improvement} and Tabel~\ref{result-improvement-imagenet} in Appendix~\ref{improve-sec} show that such a combination outperforms the original CCL-SC significantly.

\section{Conclusion}

This paper proves a generalization bound for selective classification, disclosing that optimizing feature layers to reduce intra-class variance is helpful for improving the performance. Inspired by this theory, we adapt contrastive learning to explicitly optimize the model at the feature layer, resulting in the new method CCL-SC for selective classification. Extensive experiments show that CCL-SC clearly outperforms state-of-the-art methods, and can also be naturally combined with existing techniques to bring further improvement. This work supplements previous selective classification methods which focus solely on modifying the classification layer, and might encourage the exploration of new methods considering the optimization of feature layer.

\section*{Acknowledgements}
The authors want to thank the anonymous reviewers for their helpful comments and suggestions. This work was supported by the National Science and Technology Major Project (2022ZD0116600), National Science Foundation of China (62276124, 62306104), Jiangsu Science Foundation (BK20230949), China Postdoctoral Science Foundation (2023TQ0104), Jiangsu Excellent Postdoctoral Program (2023ZB140), and the Collaborative Innovation Center of Novel Software Technology and Industrialization. 

\section*{Impact Statement}
This paper presents work whose goal is to advance the field of selective classification. There are many potential societal consequences of our work, none of which we feel must be specifically highlighted here.
\bibliography{example_paper}
\bibliographystyle{icml2024}

\appendix
\onecolumn

\section{Table of Notations}
\begin{table}[ht]
	\centering
	\caption{Key symbols and notations.}
	\vskip 0.15in
		\begin{tabular}{cp{13cm}}
			\toprule
			 \textbf{Sign} & \textbf{Description} \\
			\midrule
            $\mathcal{X}$ & The feature space.\\
            $\mathcal{Y}$ & The label space.\\
            $k$ & The number of classes.\\
			$\mathcal{D}$ & The data distribution over $\mathcal{X} \times \mathcal{Y}$.\\  
            $\mathcal{F}$ & The families of predictive probability functions mapping $\mathcal{X}$ to ${[0,1]}^{k}$.\\
            $f(\bm{x})$ & The vector composed of the predicted probabilities of $f$ for sample $\bm{x}$ across $k$ classes.    \\
            $\mathcal{G}$ & The families of confidence functions mapping $\mathcal{X}$ to $[0,1]$.\\
            $g(\bm{x})$ & The confidence of $f(\bm{x})$ estimated by the selective function $g$.\\  
            $c$ & The feature embedding layer of the classifier $f$.\\
            $l$ & The final classification layer of $f$.\\
            $\hat{y}$ & The predictive class by $f$.\\
            $\phi(g)$ & The coverage relies on the selective function $g$.\\
            $R(f, g)$ & The Selective risk of the selective model $(f, g)$.\\
            $h$ & The threshold that determines whether the model chooses to classify or not.\\
			$S$ & The training set.\\
			$m$ & The number of training samples.\\
            $P$ & The queue of all positive samples.\\
            $Q$ & The queue of all negative samples.\\
            $\bm{z}$ & The the normalized feature embedding of sample $\bm{x}$ through the embedding layer $c$.\\
            
			\bottomrule
	\end{tabular}
	\label{tab:notation}
\end{table}

\section{Theorem Proofs}\label{proofs-sec}
\textbf{Theorem.}~~$\forall \rho, \rho^{\prime}, \alpha,\beta,\lambda>0$, and $\forall \delta>0$, with probability at least $1-\delta$ over a training set of size $m$, we have:
\begin{align*}
\mathbb{E}_{(\bm{x}, y) \sim \mathcal{D}}[L_{0}(f, g, \bm{x}, y)] \leq \mathbb{E}_{(\bm{x}, y) \sim S}\left[L_{\mathrm{MH}}^{\rho, \rho^{\prime}}(f, g, \bm{x}, y)\right] +4 \sqrt{\frac{\|l\|_2^2 \operatorname{Var}_{\mathrm{intra}}[c(\bm{x})]+4 \tilde{\rho}^2+\tilde{\rho}^2\|l\|_2^2 \ln \frac{6 m}{\delta}}{\tilde{\rho}^2 m\|l\|_2^2}},
\end{align*}
where $\|l\|_2$ denotes the L2-norm of the classification layer $l$'s parameters, and $\Tilde{\rho}=\min \{\rho/(4 \alpha), \rho^{\prime}/(4 \beta \lambda+2 \alpha)\}$. 

\begin{proof}
To obtain the bound of the gap between the expected selective classification loss $\mathbb{E}_{(\bm{x}, y) \sim \mathcal{D}}[L_{0}(f, g, \bm{x}, y)]$ and the empirical margin loss $\mathbb{E}_{(\bm{x}, y) \sim S}\left[L_{\mathrm{MH}}^{\rho, \rho^{\prime}}(f, g, \bm{x}, y)\right]$, we start to prove a PAC-Bayesian bound:
\begin{lemma}\label{lem:pac}
Let $f_{\bm{w}}: \mathcal{X} \rightarrow \mathcal{Y}$ be any predictor with parameters $\bm{w}$, and $P$ be any distribution on the parameters that is independent of the training data. Then, for any $\rho, \rho^{\prime}, \alpha, \beta, \delta>0$, with probability at least $1-\delta$ over the training set of size $m$, for any $\bm{w}$, and any random perturbation $\bm{u}$ s.t. $\mathbb{P}_{\bm{u}}\left[\max _{\bm{x} \in S}\left|f_{\bm{w}+\bm{u}}(\bm{x})-f_{\bm{w}}(\bm{x})\right|_2<\right.$ $\left.\min \left\{\frac{\rho}{4 \alpha}, \frac{\rho^{\prime}}{4 \beta \lambda+2 \alpha}\right\}\right] \geq 1 / 2$, we have
\begin{equation}
\mathbb{E}_{(\bm{x}, y) \sim \mathcal{D}}[L_{0}(f, g, \bm{x}, y)] \leq \mathbb{E}_{(\bm{x}, y) \sim S}\left[L_{\mathrm{MH}}^{\rho, \rho^{\prime}}(f, g, \bm{x}, y)\right]+4 \sqrt{\frac{D_{\mathrm{KL}}(\bm{w}+\bm{u} \mid P)+\ln \frac{6 m}{\delta}}{m-1}} ,
\end{equation}
where $D_{\mathrm{KL}}(P\mid Q)$ denotes the Kullback-Leibler divergence between $P$ and $Q$.
\end{lemma}
\begin{proof}[Proof of Lemma~\ref{lem:pac}]
Let $\bm{w}^{\prime}=\bm{w}+\bm{u}$, and $\mathcal{S}_{\bm{w}}$ be the set of perturbations with the following property:
\begin{align*}
\mathcal{S}_{\bm{w}} \subseteq\Big\{\bm{w}^{\prime}\bigg| \max _{\bm{x} \in S}| f_{\bm{w}^{\prime}}(\bm{x})-\left.f_{\bm{w}}(\bm{x})|_2<\min \{\frac{\rho}{4 \alpha}, \frac{\rho^{\prime}}{4 \beta \lambda+2 \alpha}\}\right\}.
\end{align*}

Let $q$ be the probability density function over the parameters $\bm{w}^{\prime}$. We construct a new distribution $\tilde{Q}$ over predictors $f_{\tilde{\bm{w}}}$ where $\tilde{\bm{w}}$ is restricted to $\mathcal{S}_{\bm{w}}$ with the probability density function:
\begin{align*}
\tilde{q}(\tilde{\bm{w}})= \begin{cases}q(\tilde{\bm{w}}) & \text { if } \tilde{\bm{w}} \in \mathcal{S}_{\bm{w}}; \\ 0 & \text { otherwise .}\end{cases}
\end{align*}
According to the lemma assumption, we have $Z=\mathbb{P}\left[\bm{w}^{\prime} \in \mathcal{S}_{\bm{w}}\right] \geq 1 / 2$. Therefore, we can bound the change of the margins and the confidence scores for any instance:
\begin{align*}
\begin{gathered}
\max _{i, j \in[k], \bm{x} \in S}\left|\left(\left|f_{\tilde{\bm{w}}}(\bm{x})[i]-f_{\tilde{\bm{w}}}(\bm{x})[j]\right|\right)-\left(\left|f_{\bm{w}}(\bm{x})[i]-f_{\bm{w}}(\bm{x})[j]\right|\right)\right|<\frac{\rho}{2 \alpha}, \\
\max _{\bm{x} \in S}\left|g_{\tilde{\bm{w}}}(\bm{x})-g_{\bm{w}}(\bm{x})\right|<\frac{\rho^{\prime}}{4 \beta \lambda+2 \alpha}.
\end{gathered}
\end{align*}

Here we define a perturbed loss function as:
\begin{align*}
\mathbb{E}_{(\bm{x}, y) \sim \mathcal{D}}[L_{\mathrm{MH}}^{\rho / 2, \rho^{\prime} / 2}(f, g, \bm{x}, y)]=\mathbb{P}_{\mathcal{D}}\left[\gamma(\bm{x}) \leq \frac{\rho}{2 \alpha}\right] \cdot \mathbb{P}_{\mathcal{D}}\left[g(\bm{x})>-\frac{\rho^{\prime}}{4 \beta \lambda+2 \alpha}\right]+\lambda \cdot \mathbb{P}_{\mathcal{D}}\left[g(\bm{x})>\frac{\rho^{\prime}}{4 \beta \lambda+2 \alpha}\right] .
\end{align*}

We can get the following bounds:
\begin{align*}
\begin{gathered}
\mathbb{E}_{(\bm{x}, y) \sim \mathcal{D}}[L_{0}\left(f_{\bm{w}}, g_{\bm{w}}, \bm{x}, y\right)] \leq \mathbb{E}_{(\bm{x}, y) \sim \mathcal{D}}\left[L_{\mathrm{MH}}^{\rho / 2, \rho^{\prime} / 2}\left(f_{\tilde{\bm{w}}}, g_{\tilde{\bm{w}}}, \bm{x}, y\right)\right], \\
\mathbb{E}_{(\bm{x}, y) \sim S}\left[L_{\mathrm{MH}}^{\rho / 2, \rho^{\prime} / 2}\left(f_{\tilde{\bm{w}}}, g_{\tilde{\bm{w}}}, \bm{x}, y\right)\right] \leq \mathbb{E}_{(\bm{x}, y) \sim S}\left[L_{\mathrm{MH}}^{\rho, \rho^{\prime}}\left(f_{\bm{w}}, g_{\bm{w}}, \bm{x}, y\right)\right].
\end{gathered}
\end{align*}

Finally, using the proof of Lemma~1 in~\cite{pac}, with probability $1-\delta$ over the training set we have:
\begin{align*}
\begin{aligned}
\mathbb{E}_{(\bm{x}, y) \sim \mathcal{D}}[L_{0}\left(f_{\bm{w}}, g_{\bm{w}}, \bm{x}, y\right)] & \leq \mathbb{E}_{\tilde{\bm{w}}}\left[\mathbb{E}_{(\bm{x}, y) \sim \mathcal{D}}\left[L_{\mathrm{MH}}^{\rho / 2, \rho^{\prime} / 2}\left(f_{\tilde{\bm{w}}}, g_{\tilde{\bm{w}}}, \bm{x}, y\right)\right]\right] \\
& \leq \mathbb{E}_{\tilde{\bm{w}}}\left[\mathbb{E}_{(\bm{x}, y) \sim S}\left[L_{\mathrm{MH}}^{\rho / 2, \rho^{\prime} / 2}\left(f_{\tilde{\bm{w}}}, g_{\tilde{\bm{w}}}, \bm{x}, y\right)\right]\right]+2 \sqrt{\frac{2 D_{\mathrm{KL}}(\tilde{\bm{w}} \| P)+\ln \frac{2 m}{\delta}}{m-1}} \\
& \leq \mathbb{E}_{(\bm{x}, y) \sim S}\left[L_{\mathrm{MH}}^{\rho, \rho^{\prime}}\left(f_{\bm{w}}, g_{\bm{w}}, \bm{x}, y\right)\right]+2 \sqrt{\frac{2 D_{\mathrm{KL}}(\tilde{\bm{w}} \| P)+\ln \frac{2 m}{\delta}}{m-1}} \\
& \leq \mathbb{E}_{(\bm{x}, y) \sim S}\left[L_{\mathrm{MH}}^{\rho, \rho^{\prime}}\left(f_{\bm{w}}, g_{\bm{w}}, \bm{x}, y\right)\right]+4 \sqrt{\frac{D_{\mathrm{KL}}\left(\bm{w}^{\prime} \| P\right)+\ln \frac{6 m}{\delta}}{m-1}}.
\end{aligned}
\end{align*}
\end{proof}
\vspace{-3em}
Next, we will consider adding perturbation parameters $u \sim \mathcal{N}\left(0, \sigma^2 \bm{I}\right)$ to the parameters of the classification layer $\bm{w}$, i.e., $\bm{w}^{\prime}=\bm{w}+\bm{u}$, and we assume that the learned feature $c(\bm{x})$ is centered, i.e., $\mathbb{E}[c(\bm{x})]=0$. Then, we have
\begin{align*}
\begin{split}
    \E[\|\vw^\prime c(\vx)-\vw c(\vx)\|_2^2]&=\E[\|\vu\|_2^2\|c(\vx)\|_2^2]\\
    &=\sigma^2 \E[\|c(\vx)\|_2^2]\\
    &=\sigma^2 \operatorname{tr}(\E[c(\vx)c(\vx)^\top]-\E[c(\bm{x})]\E[c(\bm{x})]^\top)\\
    &=\sigma^2 \operatorname{tr}[\operatorname{Cov}[c(\vx)]]\\
    &=\sigma^2 \operatorname{tr}[\E[\operatorname{Cov}[c(\vx)|y]]+\operatorname{Cov}[\E[c(\vx)|y]]]\\
    &=\sigma^2 \operatorname{tr}\left[\sum_{i\in [k]}\operatorname{Cov}[c^i(\vx)] /k+\sum_{i\neq j}(\E[c^i(\vx)]-\E[c^j(\vx)])(\E[c^i(\vx)]-\E[c^j(\vx)])^\top/k(k-1) \right]\\
    &=\sigma^2 \operatorname{tr}\left[\sum_{i\in [k]}\operatorname{Cov}[c^i(\vx)] /k\right]+\sum_{i\neq j}\operatorname{tr}\left[(\E[c^i(\vx)]-\E[c^j(\vx)])(\E[c^i(\vx)]-\E[c^j(\vx)])^\top\right]/k(k-1)  \\
    &\leq \sigma^2 \operatorname{tr}\left[\sum_{i\in [k]}\operatorname{Cov}[c^i(\vx)] /k\right]+\sum_{i\neq j}(2\Tilde{\rho})^2/(\|l\|_2^2k(k-1)) \\
    &= \sigma^2 \operatorname{tr}\left[\sum_{i\in [k]}\operatorname{Cov}[c^i(\vx)] /k\right]+4\Tilde{\rho}^2/\|l\|_2^2.\\
\end{split}
\end{align*}
According to the Markov inequality, we have
\begin{align*}
\mathbb{P}_\beta\left[\max _{\bm{x} \in S}\left|f_{\bm{w}^{\prime}}(\bm{x})-f_{\bm{w}}(\bm{x})\right|_2^2 \geq \frac{\sigma^2}{\delta} \cdot\left(\operatorname{tr}\left[\sum_{i \in[k]} \operatorname{Cov}[c^i(\bm{x})] / k\right]+4 \tilde{\rho}^2 /\|l\|_2^2\right)\right] \leq \delta.
\end{align*}
We set $\delta=1 / 2$, such that the inequality holds with a probability at least $1 / 2$ :
\begin{align*}
\max _{\bm{x} \in S}\left|f_{\bm{w}^{\prime}}(\bm{x})-f_{\bm{w}}(\bm{x})\right|_2^2 \leq 2 \sigma^2\left(\operatorname{tr}\left[\sum_{i \in[k]} \operatorname{Cov}\left[c^i(\bm{x})\right] / k\right]+4 \tilde{\rho}^2 /\|l\|_2^2\right).
\end{align*}

For simplicity, we use $\operatorname{Var}_{\text {intra }}[c(\bm{x})]$ to denote the intra-class variance $\operatorname{tr}\left[\sum_{i \in[k]} \operatorname{Cov}\left[c^i(\bm{x})\right] / k\right]$, then we have
\begin{align*}
\max _{\bm{x} \in S}\left|f_{\bm{w}^{\prime}}(\bm{x})-f_{\bm{w}}(\bm{x})\right|_2^2 \leq 2 \sigma^2\left(\operatorname{Var}_{\text {intra }}[c(\bm{x})]+4 \tilde{\rho}^2 /\|l\|_2^2\right).
\end{align*}

Since we now prove that the perturbation caused by random vector $\bm{u}$ is bounded by a term relative to the variance $\sigma$, we can preset the value of $\sigma$ to make the random perturbation satisfy the condition for Lemma~\ref{lem:pac}.
\begin{align*}
\begin{aligned}
\max _{x \in S}\left|f_{{\boldsymbol{w}}^{\prime}}(\bm{x})-f_{\boldsymbol{w}}(\bm{x})\right|_2^2 & \leq 2 \sigma^2\left(\operatorname{Var}_{\text {intra }}[c(\bm{x})]+4 \tilde{\rho}^2 /\|l\|_2^2\right) \\
& =\min \left\{\frac{\rho}{4 \alpha}, \frac{\rho^{\prime}}{4 \beta \lambda+2 \alpha}\right\}^2 \\
& =\tilde{\rho}^2.
\end{aligned}
\end{align*}

We can derive $\sigma=\tilde{\rho} / \sqrt{2\left(\operatorname{Var}_{\text {intra }}[c(\bm{x})]+4 \tilde{\rho}^2 /\|l\|_2^2\right)}$ from the above inequality. Naturally, we can calculate the Kullback-Leibler divergence in Lemma~\ref{lem:pac} with the chosen distributions for $P \sim \mathcal{N}\left(0, \sigma^2 \bm{I}\right)$:
$$
D_{\mathrm{KL}}(\boldsymbol{w}+\boldsymbol{u} \| P) \leq \frac{|\boldsymbol{w}|^2}{2|\boldsymbol{w}|^2 \sigma^2}=\frac{1}{2 \sigma^2} \leq \frac{\|l\|_2^2 \operatorname{Var}_{\text {intra }}[c(\boldsymbol{x})]+4 \tilde{\rho}^2}{\tilde{\rho}^2\|l\|_2^2}.
$$

Put it in Lemma~\ref{lem:pac}, for any $\boldsymbol{w}$, with probability of at least $1-\delta$ we have:
$$
\mathbb{E}_{(\bm{x}, y) \sim \mathcal{D}}[L_{0}(f, g, \bm{x}, y)]  \leq \mathbb{E}_{(\boldsymbol{x}, y) \sim S}\left[L_{\mathrm{MH}}^{\rho, \rho^{\prime}}(f, g, \boldsymbol{x}, y)\right]+4 \sqrt{\frac{\|l\|_2^2 \operatorname{Var}_{\text {intra }}[c(\boldsymbol{x})]+4 \tilde{\rho}^2+\tilde{\rho}^2\|l\|_2^2 \ln \frac{6 m}{\delta}}{\tilde{\rho}^2 m\|l\|_2^2}} .
$$
\end{proof}

\section{Training with CSC Loss}\label{train-sec}
\begin{algorithm}[ht]
    \caption{Training with CSC loss}
    \label{alg:algorithm}
    \textbf{Input}: Data $\left\{\left(\bm{x}_{i},  y_{i}\right)\right\}_{i=1}^{m}$, initial model $f$  \\
    \textbf{Parameter}: momentum coefficient $q$, queue size $s$, weight coefficient $w$, initial epochs $E_s$
    \begin{algorithmic}[1] 
        \FOR{$e=1: $ maximum epochs}
            \IF {$e = E_s$}
                \STATE Initialize $\bm{\theta}_\mathrm{m} = \bm{\theta}$
            \ENDIF
            \FOR{each mini-batch in the current epoch $e$}
                \IF {$e \geq E_s$}
                \IF {\;Both queues $P$ and $Q$ have been updated more than $s$ tuples}
                \STATE $L = w \cdot L_\text{CSC} + L_\text{CE}$
                \ELSE
                \STATE $L = L_\text{CE}$
                \STATE Fetch $\left\{\left(\bm{z}_{i},  f_{\bm{\theta}_\mathrm{m}}(\bm{x}_i)\right)\right\}_{i=1}^{n}$ of mini-batch data;
                \STATE Update $P$ by $\left\{ (\bm{z}_i, f_{\bm{\theta}_\mathrm{m}}(\bm{x}_i)) \mid f_{\bm{\theta}_\mathrm{m}}(\bm{x}_i) = y_i \right\}_{i=1}^{n}$;
                \STATE Update $Q$ by $\left\{ (\bm{z}_i, f_{\bm{\theta}_\mathrm{m}}(\bm{x}_i)) \mid f_{\bm{\theta}_\mathrm{m}}(\bm{x}_i) \neq y_i \right\}_{i=1}^{n}$
                \ENDIF
                \STATE $\bm{\theta}_{\mathrm{m}} \leftarrow q \cdot \bm{\theta}_{\mathrm{m}}+(1-q) \cdot \bm{\theta}$;
                \ELSE
                \STATE $L = L_\text{CE}$
                \ENDIF
                \STATE Update the parameters $\bm{\theta}$ of the model $f$ using an optimizer based on $L$
            \ENDFOR
        \ENDFOR
    \end{algorithmic}
\end{algorithm}
Algorithm~\ref{alg:algorithm} provides the pseudo-code of the training method with the proposed CSC loss. Different from the conventional two-stage training manner ``pre-train then finetune'' used in typical contrastive learning methods~\cite{moco,simclr,supcon}, we employ a one-stage manner to optimize the classification layers and convolutional layers of the model together, because the CSC loss involves the predictive information of the model. 

The model $f$ is initially trained using the cross-entropy loss $L_\text{CE}$ for $E_s$ epochs. When $e \geq E_s$ and meanwhile both queues $P$ and $Q$ have been updated more than $s$ tuples, the training gradient consists of two parts: the gradient of the CSC loss $L_\text{CSC}$ returned from the last feature embedding layer $c$, and the gradient of the cross-entropy loss $L_\text{CE}$ returned from the classification layer $l$. We use a weight coefficient $w$ to balance these two loss items. The parameters $\bm{\theta}$ of the model $f$ will be updated using an optimizer based on the combined loss. As introduced in Section~\ref{sec-samples}, when $e \geq E_s$, a momentum encoder ${f_{\bm{\theta}}}_{\mathrm{m}}$ is used to generate positive and negative samples. When the epoch $e$ equals $E_s$, the parameters $\bm{\theta}_\mathrm{m}$ of the momentum encoder ${f_{\bm{\theta}}}_{\mathrm{m}}$ are initialized to the parameters $\bm{\theta}$ of the current model $f$, which ensures that the momentum encoder has a favorable initial accuracy. After that, the parameters $\bm{\theta}_\mathrm{m}$ of the momentum encoder will be updated according to Eq.~\eqref{update} at each training step. The samples generated by the momentum encoder will be used to update the queues $P$ and $Q$, which store positive and negative samples, respectively. Note that ${f_{\bm{\theta}}}_{\mathrm{m}}$ will not be optimized by the optimizer. 

\section{Detailed Experimental Settings}\label{setting sec}
\subsection{hyper-parameters}
\label{param sec}
We utilize 20\% of the training set for each dataset as the validation set to tune hyper-parameters.  We test the momentum coefficient $q\in \{0.9 0.99, 0.999\}$, and the weight coefficient $w\in \{0.1, 0.5, 1.0\}$. For the queue size $s$, we set it based on the number of classes in the dataset: for datasets with fewer classes such as CelebA and CIFAR-10, we set $s=300$; for datasets with more classes such as CIFAR-100 and ImageNet, we set $s=3000$ and $s=10000$, respectively. For the initial epochs $E_s$, we test it on datasets with total training epochs of 300 using the values~$\{50, 100, 150, 200\}$; for datasets with total training epochs of 150, we test the values~$\{50, 100\}$; for datasets with training epochs of 50, we evaluate the values~$\{0, 1\}$. Table~\ref{param} lists the hyperparameter settings of CCL-SC on each dataset. After tuning the hyper-parameters, we train the model on the entire training set to evaluate performance.
\begin{table}[h]
\caption{The hyper-parameters settings of CCL-SC on various datasets.}\vskip 0.15in
\centering
    \begin{tabular}{ccccc}
    \toprule
    Dataset & $q$     & $s$     & $w$ & $E_s$ \\
    \midrule
    CIFAR-10 & 0.999 & 300   & 0.5   & 150 \\
    CIFAR-100 & 0.99  & 3000  & 1.0     & 150 \\
    CelebA & 0.999 & 300   & 0.5   & 1 \\
    ImageNet & 0.999  & 10000 & 0.1   & 50 \\\bottomrule
    \end{tabular}%
\label{param}
\end{table}

For the hyper-parameters of the baseline methods, we follow the settings provided in their origin paper or released codes. Nevertheless, due to the absence of performance evaluation on CIFAR-100 and CelebA in prior work, we have also applied the same parameter-tuning steps outlined for our method to calibrate the parameters of the baseline methods. Table~\ref{param-baseline} lists the hyperparameter settings for each baseline on CIFAR-100 and CelebA.
\begin{table}[h]
\caption{The hyper-parameters settings of the baselines on various datasets.}\vskip 0.15in
\centering
\begin{tabular}{ccc}\toprule 
    Dataset & Method & hyper-parameters \\
    \midrule
    \multirow{3}[2]{*}{CIFAR-100} & Deep Gambler & Initial epochs $E_s$ = 200,  Reward $o$ = 4.6 \\
          & SAT   & Initial epochs $E_s$ = 200, Momentum term $m_\text{SAT}$ = 0.9 \\
          & SAT+ER & Initial epochs $E_s$ = 200, Momentum term $m_\text{SAT}$ = 0.9, Entropy weight $\beta$ = 0.001 \\
    \midrule
    \multirow{3}[2]{*}{CelebA} & Deep Gambler & Initial epochs $E_s$ = 0,  Reward $o$ = 2.0 \\
          & SAT   & Initial epochs $E_s$ = 0, Momentum term $m_\text{SAT}$ = 0.9 \\
          & SAT+ER & Initial epochs $E_s$ = 0, Momentum term $m_\text{SAT}$ = 0.9, Entropy weight $\beta$ = 0.01 \\
    \bottomrule
    \end{tabular}%

\label{param-baseline}
\end{table}
\subsection{Networks and Training}\label{network sec} 
Following prior work, for CIFAR-10 and CIFAR-100, we use VGG16~\cite{vgg} as the backbones of selective classifiers. The models are trained for 300 epochs using SGD, with an initial learning rate of 0.1, a momentum of 0.9, a weight decay of 5e-4, and a mini-batch size of 64. The learning rate was reduced by 0.5 every 25 epochs. 

For ImageNet, we use ResNet34~\cite{resnet} trained for 150 epochs using SGD, with an initial learning rate of 0.1, a momentum of 0.9, a weight decay of 5e-4, and a mini-batch size of 256. The learning rate was reduced by 0.5 every 10 epochs. 

For CelebA, we use ResNet18~\cite{resnet} trained for 50 epochs using Adam, with an initial learning rate of 1e-5, and a mini-batch size of 64. When evaluating the performance of each method on the CelebA dataset, we employ the checkpoint with the highest accuracy on the CelebA's original separate validation set to assess its performance on the test set.
\subsection{Data Augmentation Methods for each Dataset}
\label{Data Augmentation sec}
For the same dataset, all compared methods utilize the same data augmentation to ensure a fair comparison. For the commonly used selective classification benchmark datasets (including CIFAR-10, and ImageNet) in previous works~\cite{SelectiveNet,DeepGambler,NIPS:SAT,SAT,Entropy+SR}, we adopt the data augmentation settings that have been commonly utilized. For the newly considered challenging selective classification datasets in this work, CIFAR-100 and CelebA, we apply commonly used data augmentation methods from the field of image classification that are tailored for these datasets. Table~\ref{data-augmentation} presents a summary of the data augmentation methods utilized for each of the datasets.
\vspace{-0.8em}
\begin{table}[h]
\caption{Data augmentation methods utilized for each of the datasets.}\vskip 0.15in
\centering
    \begin{tabular}{cc}
    \toprule 
    Dataset & Data Augmentation \\
    \midrule
    \multirow{2}[2]{*}{CIFAR-10} & RandomCrop \\
          & RandomHorizontalFlip \\
    \midrule
    \multirow{2}[2]{*}{CIFAR-100} & RandomCrop \\
          & RandomHorizontalFlip \\
    \midrule
    CelebA & RandomHorizontalFlip \\
    \midrule
    \multirow{3}[2]{*}{ImageNet} & RandomResizedCrop \\
          & RandomHorizontalFlip \\
          & ColorJitter \\
    \bottomrule
    \end{tabular}%

\label{data-augmentation}
\end{table}
\vspace{-0.8em}
\section{Comparison with Other Related Methods} \label{ref-compare-other}
As described in Section~\ref{sec-sc related}, selective classification is closely related to model calibration and the Human-AI collaboration. In this section, we introduce five methods from model calibration, including Focal loss~\cite{focal_loss}, Adaptive Focal loss~\cite{afocalloss}, Soft AvUC loss~\cite{softavucece}, Soft ECE loss~\cite{softavucece}, and MMCE loss~\cite{mmce}, as well as AUCOC loss~\cite{aucoc} from the Human-AI collaboration into selective classification. Tabel~\ref{result-other-method} shows that CCL-SC has the lowest selective risk at various coverage rates compared to these methods.

\vspace{-0.8em}
\begin{table*}[h!]
\caption{Selective risk (\%) on the CIFAR-100 for various coverage rates (\%). The mean and standard deviation are calculated over 5 trials. The best entries are marked in bold.} \vskip 0.15in
\small
\fontsize{9}{10}\selectfont
\centering
    \begin{tabular}{cccccccc}
    \toprule 
    
Coverage & CCL-SC & Adaptive Focal & Soft AvUC & Soft-ECE & MMCE   & AUCOC Loss & Focal Loss  \\
\midrule
100   & \textbf{26.55±0.26} & 27.96±0.12 & 27.93±0.12 & 26.81±0.06 & 27.14±0.24 & 26.75±0.07 & 28.08±0.22 \\
    95    & \textbf{23.54±0.15} & 25.26±0.24 & 24.99±0.05 & 23.93±0.16 & 24.30±0.24 & 23.89±0.05 & 25.38±0.26 \\
    90    & \textbf{20.97±0.20} & 22.63±0.17 & 22.20±0.08 & 21.40±0.1 & 21.66±0.12 & 21.33±0.19 & 22.64±0.37 \\
    85    & \textbf{18.57±0.20} & 20.13±0.09 & 19.68±0.04 & 18.84±0.10 & 18.95±0.25 & 19.17±0.10 & 20.11±0.43 \\
    80    & \textbf{16.07±0.15} & 17.77±0.08 & 17.12±0.05 & 16.35±0.03 & 16.35±0.19 & 16.21±0.05 & 17.64±0.18 \\
    75    & \textbf{13.60±0.19} & 15.48±0.14 & 14.76±0.01 & 14.15±0.11 & 14.21±0.13 & 13.95±0.04 & 15.34±0.14 \\
    70    & \textbf{11.23±0.16} & 13.07±0.11 & 12.46±0.10 & 11.57±0.07 & 11.60±0.16 & 11.47±0.13 & 13.17±0.11 \\
    60    & \textbf{6.83±0.15} & 8.73±0.01 & 7.93±0.13 & 7.28±0.08 & 7.32±0.05 & 7.26±0.13 & 8.85±0.12 \\
    50    & \textbf{3.95±0.22} & 5.67±0.13 & 4.29±0.05 & 4.51±0.07 & 4.32±0.10 & 4.21±0.02 & 5.69±0.11 \\
    40    & \textbf{2.29±0.33} & 4.14±0.29 & 2.48±0.02 & 3.10±0.02 & 3.16±0.11 & 2.31±0.09 & 4.59±0.01 \\
    30    & \textbf{1.26±0.17} & 3.87±0.07 & 1.63±0.03 & 2.72±0.15 & 2.63±0.03 & 1.69±0.04 & 4.25±0.05 \\
    20    & \textbf{0.71±0.12} & 3.95±0.25 & 1.58±0.03 & 2.85±0.20 & 2.40±0.05 & 1.26±0.16 & 4.70±0.30 \\
    10    & \textbf{0.36±0.08} & 4.55±0.45 & 1.45±0.15 & 2.85±0.15 & 2.20±0.00 & 0.88±0.04 & 5.05±0.25 \\\bottomrule
    \end{tabular}
\label{result-other-method}
\end{table*}
\vspace{-0.8em}
\section{Alignment between Proposed Theory and Method}\label{sec-theoty-method}
In this section, we show the changes of the intra-class variance and the bound in \mbox{Theorem~\ref{thm:bound}} of different methods during the training process on CIFAR-100 in Figure~\ref{fig:theorem-method}. It can be observed that CCL-SC does have the lowest intra-class variance and the lowest bound. It is worth mentioning that the relative order of the intra-class variance and the bounds of different methods is consistent with their actual order of selective risk, and the methods with similar selective risk (such as SAT+EM and SAT) also have similar intra-class variance/bound, indicating the importance of intra-class variance for selective classification performance and the usefulness of our bound.
\vspace{-0.8em}
\begin{figure*}[h]
	\centering
	\mbox{
		\subfigure[\label{subfig:a-var} Intra-variance]{\includegraphics[width=0.49\linewidth]{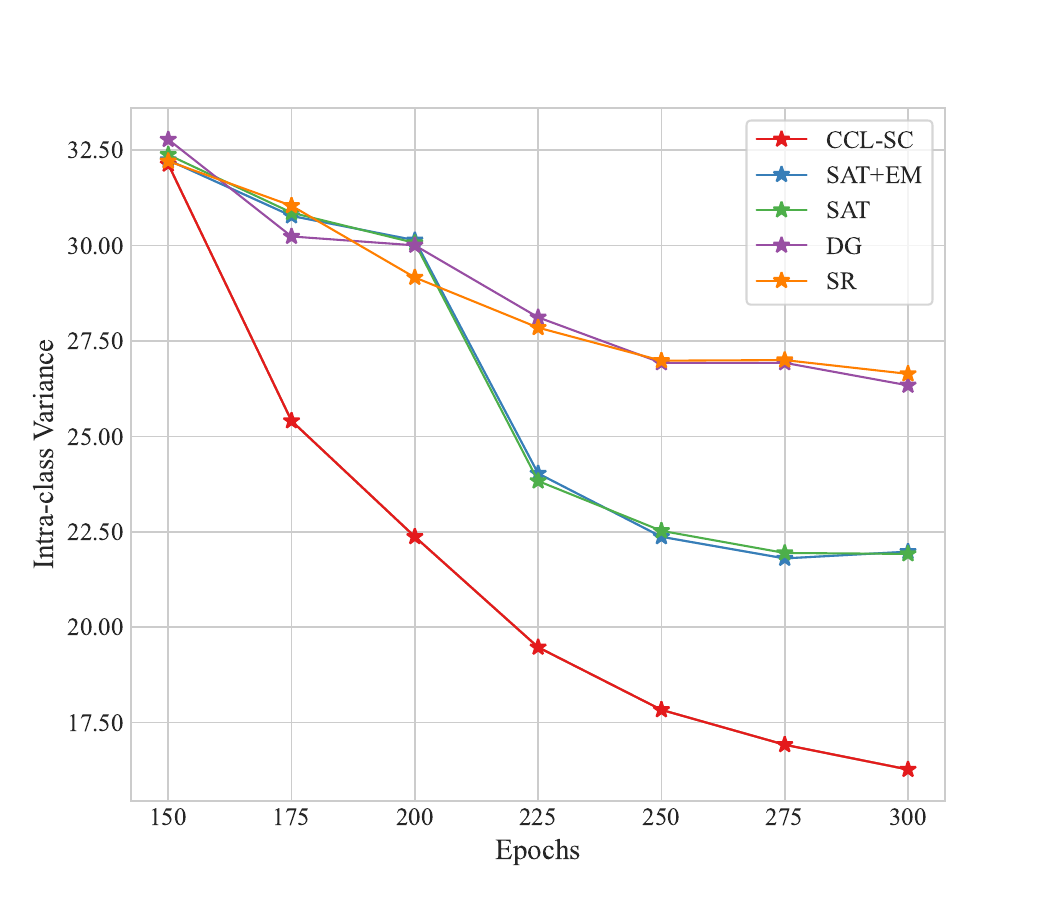}}
		\subfigure[\label{subfig:b-bound} Bound in \mbox{Theorem~\ref{thm:bound}}]{\includegraphics[width=0.49\linewidth]{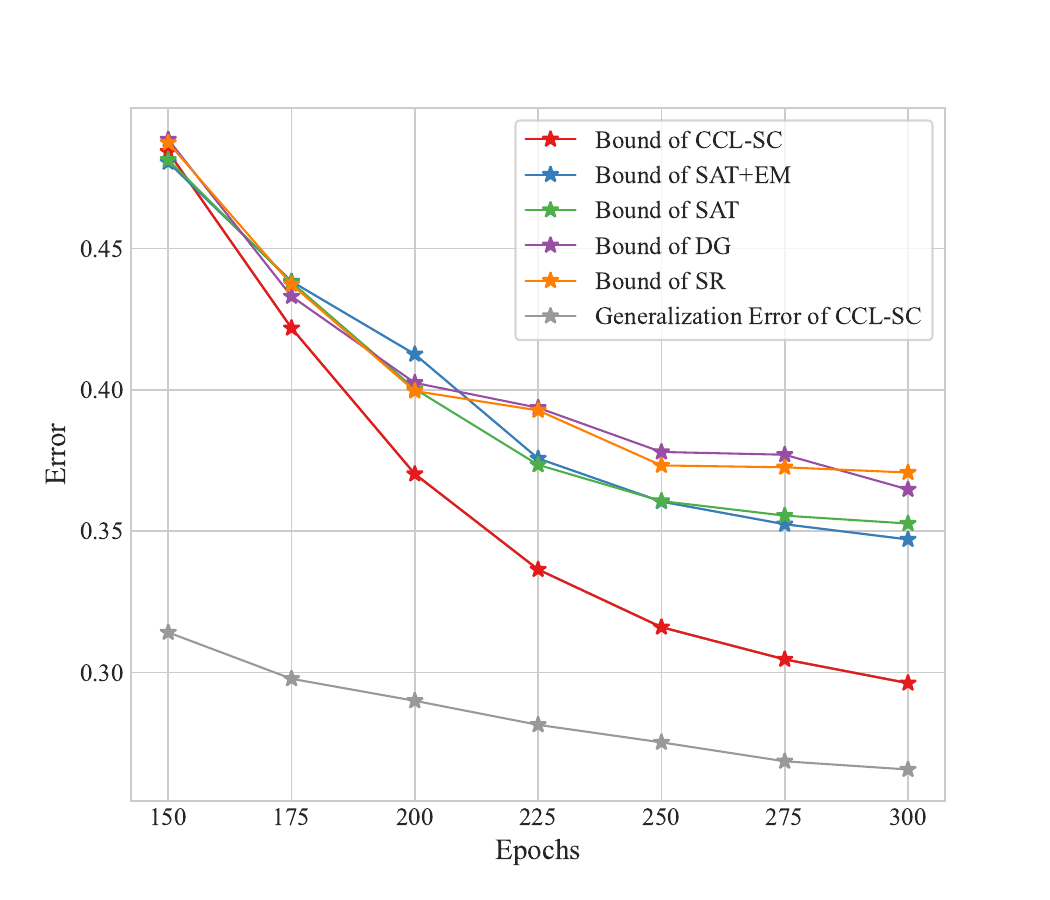}}
	}\vspace{-0em}
	\caption{The intra-class variance (a) and the bound in \mbox{Theorem~\ref{thm:bound}} (b) changes of different methods during the training process on CIFAR-100. In (b), we also include the generalization error of CCL-SC.}
	\label{fig:theorem-method}
\end{figure*}
\vspace{-0.8em}

\section{Learned Feature Representation}\label{tsne-sec}
In this section we compare its learned feature representations with those of SR trained only using the cross-entropy loss on CIFAR-10 at coverage 95\%. The t-SNE~\cite{t-SNE} visualization shown in Figure~\ref{tSNE-cifar10} clearly demonstrates that compared to SR, our method CCL-SC achieves more significant inter-class separation and intra-class aggregation in the feature space for selecting samples for classification. This confirms that optimizing the feature layer contributes to performance improvement in selective classification.

\begin{figure*}[h]
\vspace{-0.8em}
	\centering
	\mbox{
		\subfigure[\label{subfig:a-cifar10} SR]{\includegraphics[width=0.49\linewidth]{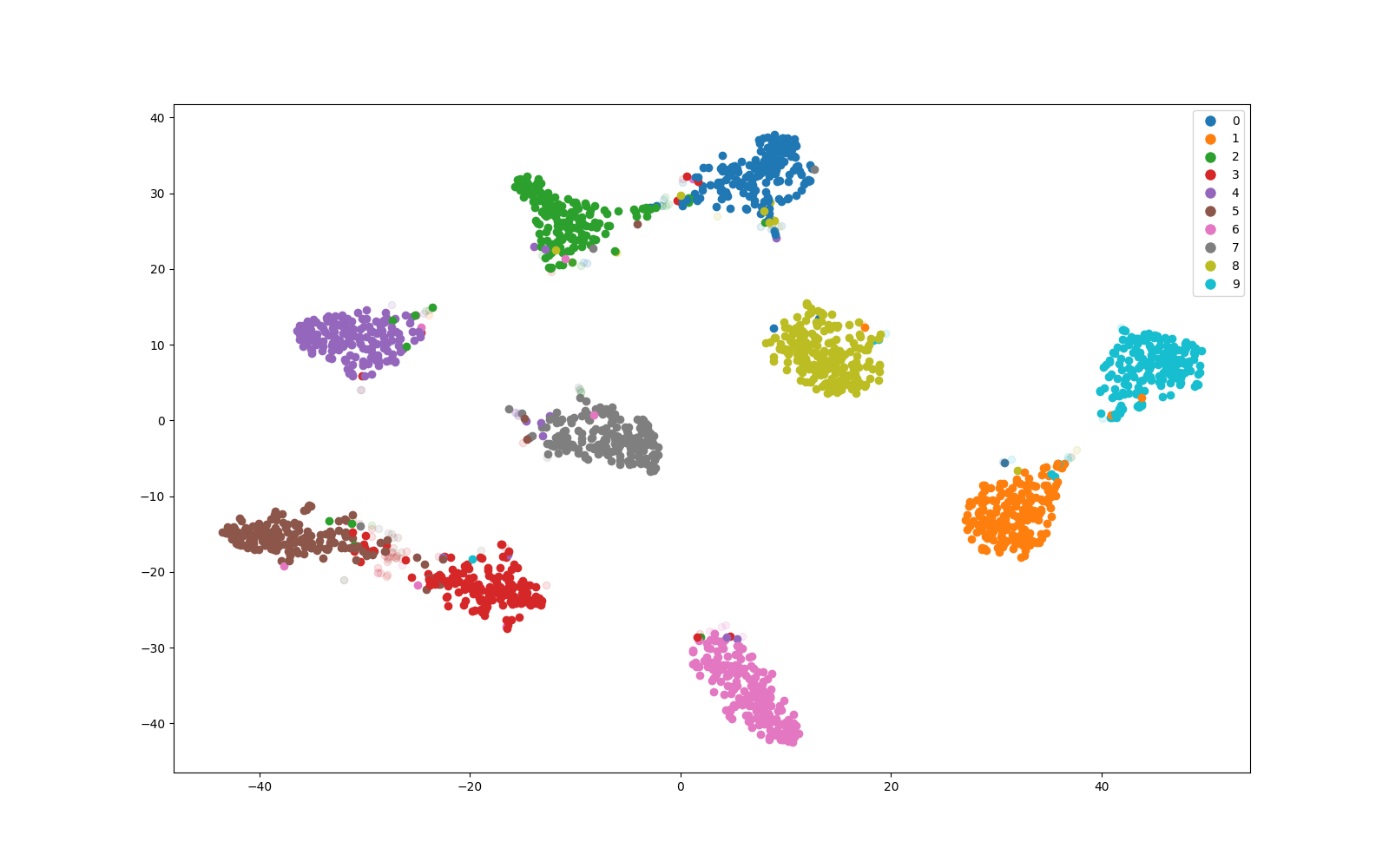}}
		\subfigure[\label{subfig:b-cifar10} CCL-SC]{\includegraphics[width=0.49\linewidth]{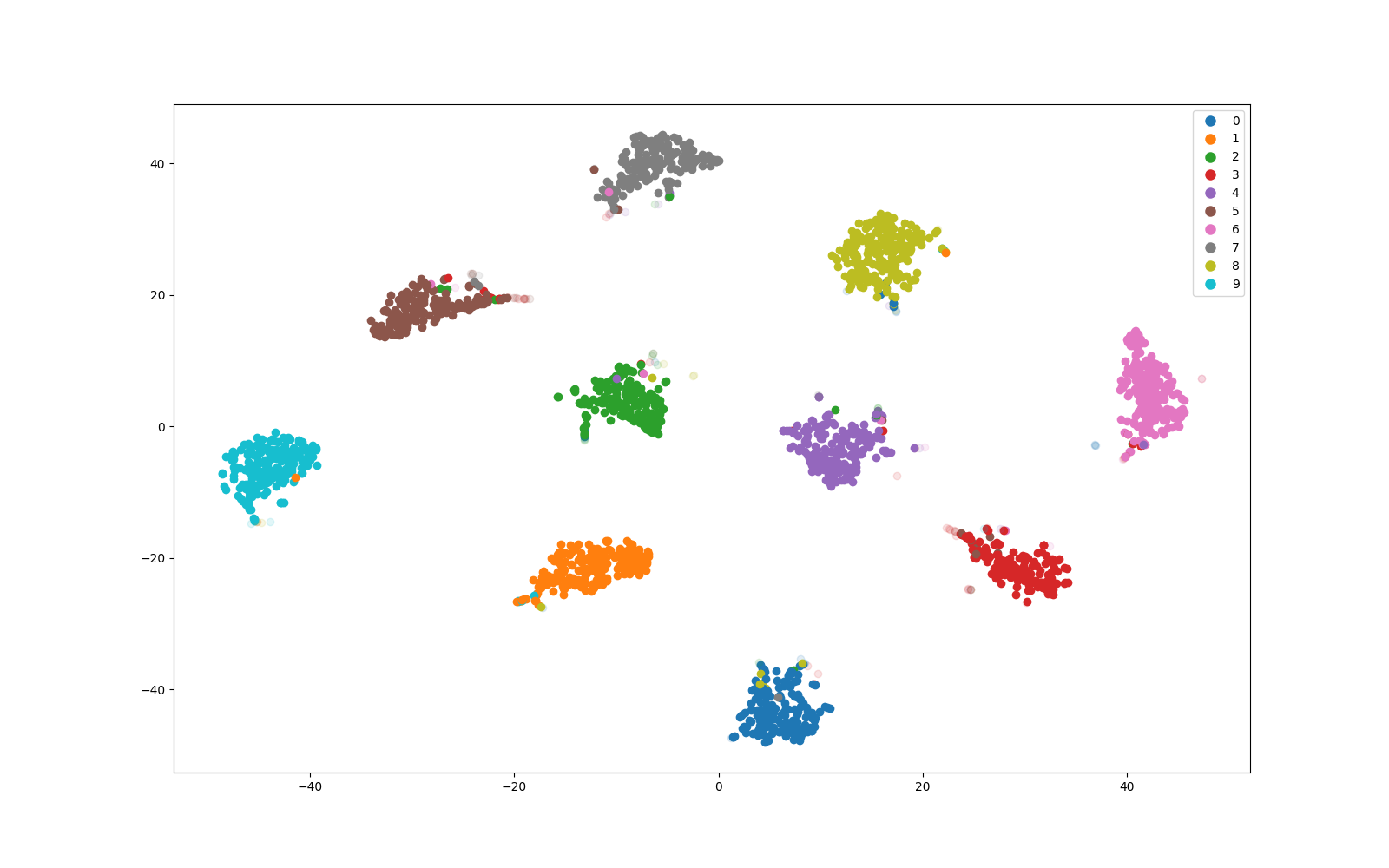}}
	}\vspace{-0em}
	\caption{The t-SNE Visualization~\cite{t-SNE} of SR and CCL-SC feature representations on the CIFAR-10 dataset at 95\% coverage. Point colors indicate class categories. Light-colored points represent samples selected for abstaining from predicting.}
	\label{tSNE-cifar10}\vspace{-0.5em}
\end{figure*}

\section{Ablation Study and Hyper-parameter Sensitivity Results}\label{ab-sec}
\textbf{SR-weighted} To verify whether the weighting manner based on SR in the proposed CSC loss $L_{\text{CSC}}$ really improves the performance of selective classification, we evaluate the model trained using CSC loss without applying weights (i.e., $L^{\prime}_{\text {CSC}} =  \frac{1}{-|P(y)|} \sum_{\bm{x}_p \in P(y)} \log \frac{\exp \left(\bm{z} \cdot \bm{z}_p / \tau\right)}{ \sum_{\bm{x}_a \in A(y)} \exp \left(\bm{z} \cdot \bm{z}_a / \tau\right)}\nonumber$), and the results on the CIFAR-100 dataset are shown in Table~\ref{result-Ablation}. It can be observed that SR-weighted (i.e., original CCL-SC method) consistently achieves lower selective risk than using the unweighted CSC loss across all degrees of coverage, and performs significantly better on at $5/13$ coverage rates, according to the Wilcoxon rank-sum test~\cite{rank-sum} with significance level 0.05. This result demonstrates the effectiveness of CSC loss combined with model confidence for selective classification problems.
\begin{table}[h]

\caption{Selective risk (\%) on the CIFAR-100 dataset for various coverage rates (\%). The means and standard deviations are calculated over 5 trials. The best entries are marked in bold. The symbol `$\bullet$'/`$\circ$' indicates that SR-weighted (i.e., original CCL-SC) is significantly better/worse than Unweighted, according to the Wilcoxon rank-sum test with significance level 0.05.}\vskip 0.15in
\centering
\begin{tabular}{ccc}
\toprule 
    Coverage & SR-weighted & Unweighted\\
    \midrule
    100   & \textbf{26.55±0.26} & 26.72±0.15\ \  \\
    95    & \textbf{23.54±0.15} & 23.84±0.11• \\
    90    & \textbf{20.97±0.20} & 21.28±0.16• \\
    85    & \textbf{18.57±0.20} & 18.69±0.21\  \ \\
    80    & \textbf{16.07±0.15} & 16.15±0.17\ \ \\
    75    & \textbf{13.60±0.19} & 13.67±0.16\ \  \\
    70    & \textbf{11.23±0.16} & 11.40±0.20\ \  \\
    60    & \textbf{6.83±0.15} & \ \ 7.31±0.14• \\
    50    & \textbf{3.95±0.22} & \ \ 4.52±0.13• \\
    40    & \textbf{2.29±0.33} & 2.54±0.23\ \\
    30    & \textbf{1.26±0.17} & 1.47±0.22\ \\
    20    & \textbf{0.71±0.12} & 0.82±0.09\  \\
    10    & \textbf{0.36±0.08} & \ \ 0.58±0.12• \\\bottomrule
    \end{tabular}%
\label{result-Ablation}
\end{table}

\textbf{The contrastive learning method of CCL-SC}  We conduct experiments with different queue sizes $s=300$ and $s=3000$. We first conduct ablation experiments for the construction of negative samples. For convenience, we name the ablation method CCL-SC2. For the negative samples of the samples that are correctly classified as class $y$, CCL-SC2 contains not only the samples misclassified as class $y$ in the queue defined in this paper, but also the samples from other classes in the queue. Table~\ref{result-ccl-sc2} shows that even with the addition of negative samples from other categories, the performance of CCL-SC2 will not be improved compared to the original CCL-SC. However, this leads to more training costs. This implies that the real improvement in model performance is likely to be from the negative samples we define. 

\begin{table}[h]
\caption{Selective risk (\%) on the CIFAR-100 dataset for various coverage rates (\%). The means and standard deviations are calculated over 5 trials. The best entries are marked in bold.}\vskip 0.15in
\centering
\begin{tabular}{ccc|cc}
\toprule 
    Coverage & CCL-SC $ s=300$ & CCL-SC2 $s=300$ & CCL-SC $s=3000$ & CCL-SC2 $s=3000$\\
    \midrule
    100   & 26.73±0.21 & \textbf{26.72±0.28} & 26.55±0.26 & \textbf{26.48±0.27} \\
    95    & 23.89±0.16 & \textbf{23.87±0.18} & \textbf{23.54±0.15} & 23.74±0.29 \\
    90    & \textbf{21.22±0.15} & \textbf{21.22±0.12} & \textbf{20.97±0.20} & 21.12±0.26 \\
    85    & \textbf{18.76±0.17} & 18.85±0.10 & \textbf{18.57±0.20} & 18.67±0.24 \\
    80    & \textbf{16.36±0.21} & 16.37±0.08 & \textbf{16.07±0.15} & 16.20±0.28 \\
    75    & \textbf{13.88±0.18} & 14.05±0.10 & \textbf{13.60±0.19} & 13.79±0.27 \\
    70    & \textbf{11.37±0.23} & 11.58±0.19 & \textbf{11.23±0.16} & 11.42±0.30 \\
    60    & \textbf{7.09±0.10} & 7.30±0.19 & \textbf{6.83±0.15} & 7.50±0.28 \\
    50    & \textbf{4.06±0.23} & 4.36±0.18 & \textbf{3.95±0.22} & 4.68±0.18 \\
    40    & \textbf{2.29±0.18} & 2.47±0.23 & \textbf{2.29±0.33} & 2.62±0.02 \\
    30    & \textbf{1.24±0.08} & 1.32±0.31 & \textbf{1.26±0.17} & 1.44±0.21 \\
    20    & 0.78±0.14 & \textbf{0.73±0.17} & \textbf{0.71±0.12} & 0.97±0.12 \\
    10    & 0.50±0.17 & \textbf{0.49±0.12} & \textbf{0.36±0.08} & 0.58±0.13 \\\bottomrule
    \end{tabular}%
\label{result-ccl-sc2}
\end{table}

To confirm this conclusion, we conduct another ablation study in which we remove the negative samples defined in CCL-SC, and only use randomly sampled samples from other categories as negative samples. We name this ablation method CCL-SC3. The experimental results are shown in Table~\ref{result-ccl-sc3}. It can be observed that CCL-SC3 has a significant performance decrease compared to CCL-SC, which demonstrates the effectiveness of our strategy to construct negative samples.

\begin{table}[h]
\caption{Selective risk (\%) on the CIFAR-100 dataset for various coverage rates (\%). The means and standard deviations are calculated over 5 trials. The best entries are marked in bold.}\vskip 0.15in
\centering
\begin{tabular}{ccc|cc}
\toprule 
    Coverage & CCL-SC $ s=300$ & CCL-SC3 $s=300$ & CCL-SC $s=3000$ & CCL-SC3 $s=3000$\\
    \midrule
100   & \textbf{26.73±0.21} & 27.11±0.09 & \textbf{26.55±0.26} & 26.87±0.35 \\
    95    & \textbf{23.89±0.16} & 24.28±0.14 & \textbf{23.54±0.15} & 23.94±0.41 \\
    90    & \textbf{21.22±0.15} & 21.70±0.17 & \textbf{20.97±0.20} & 21.33±0.31 \\
    85    & \textbf{18.76±0.17} & 19.24±0.20 & \textbf{18.57±0.20} & 18.89±0.28 \\
    80    & \textbf{16.36±0.21} & 16.76±0.21 & \textbf{16.07±0.15} & 16.41±0.21 \\
    75    & \textbf{13.88±0.18} & 14.27±0.19 & \textbf{13.60±0.19} & 13.96±0.14 \\
    70    & \textbf{11.37±0.23} & 11.86±0.25 & \textbf{11.23±0.16} & 11.58±0.17 \\
    60    & \textbf{7.09±0.10} & 7.54±0.20 & \textbf{6.83±0.15} & 7.34±0.37 \\
    50    & \textbf{4.06±0.23} & 4.13±0.12 & \textbf{3.95±0.22} & 4.04±0.17 \\
    40    & 2.29±0.18 & \textbf{2.13±0.04} & 2.29±0.33 & \textbf{2.19±0.11} \\
    30    & 1.24±0.08 & \textbf{1.14±0.14} & 1.26±0.17 & \textbf{1.15±0.08} \\
    20    & 0.78±0.14 & \textbf{0.72±0.05} & \textbf{0.71±0.12} & 0.76±0.10 \\
    10    & 0.50±0.17 & \textbf{0.44±0.08} & \textbf{0.36±0.08} & 0.40±0.14 \\\bottomrule
    \end{tabular}%
\label{result-ccl-sc3}
\end{table}
\vspace{-1.0em}
We also conduct ablation experiments for the whole contrastive method of CCL-SC. Specifically, We introduce the positive and negative sample definition method and loss function from \cite{supcon} into our CCL-SC method, while keeping the other components consistent. We name this ablation method CCL-SC+SupCon. The comparison results on CIFAR-100 are shown in Table~\ref{result-ccl-sc3}. It can be observed that the selective classification performance of the original CCL-SC is better than CCL-SC+SupCon, which uses vanilla supervised CL.

\begin{table}[h]
\caption{Selective risk (\%) on the CIFAR-100 dataset for various coverage rates (\%). The means and standard deviations are calculated over 5 trials. The best entries are marked in bold.}\vskip 0.15in
\centering
\begin{tabular}{ccc|cc}
\toprule 
        Coverage & CCL-SC $ s=300$ & CCL-SC+SupCon $s=300$ & CCL-SC $s=3000$ & CCL-SC+SupCon $s=3000$\\
    \midrule
100   & \textbf{26.73±0.21} & 26.99±0.13 & \textbf{26.55±0.26} & 26.77±0.17 \\
    95    & \textbf{23.89±0.16} & 24.14±0.06 & \textbf{23.54±0.15} & 23.91±0.14 \\
    90    & \textbf{21.22±0.15} & 21.49±0.12 & \textbf{20.97±0.20} & 21.31±0.15 \\
    85    & \textbf{18.76±0.17} & 19.10±0.13 & \textbf{18.57±0.20} & 18.80±0.17 \\
    80    & \textbf{16.36±0.21} & 16.64±0.24 & \textbf{16.07±0.15} & 16.36±0.22 \\
    75    & \textbf{13.88±0.18} & 14.19±0.19 & \textbf{13.60±0.19} & 13.98±0.26 \\
    70    & \textbf{11.37±0.23} & 11.74±0.21 & \textbf{11.23±0.16} & 11.63±0.09 \\
    60    & \textbf{7.09±0.10} & 7.26±0.16 & \textbf{6.83±0.15} & 7.24±0.10 \\
    50    & \textbf{4.06±0.23} & 4.15±0.07 & \textbf{3.95±0.22} & 4.13±0.03 \\
    40    & 2.29±0.18 & \textbf{2.25±0.20} & \textbf{2.29±0.33} & 2.32±0.09 \\
    30    & \textbf{1.24±0.08} & 1.25±0.17 & \textbf{1.26±0.17} & 1.27±0.11 \\
    20    & \textbf{0.78±0.14} & 0.82±0.13 & \textbf{0.71±0.12} & 0.87±0.10 \\
    10    & 0.50±0.17 & \textbf{0.48±0.10} & \textbf{0.36±0.08} & 0.50±0.09 \\\bottomrule
    \end{tabular}%
\label{result-supcon}
\end{table}

We then conduct sensitivity analyses on the hyper-parameters in our method. Specifically, when varying one hyperparameter, we keep the other hyper-parameters fixed.

\textbf{Momentum coefficient $\bm{q}$}. Tabel~\ref{result-m} shows the influence of momentum coefficient $q \in \{0, 0.9, 0.99, 0.999\}$ on the selective classification performance of our method. Our method achieves stable selective risk when employing $q$ values in $\{0.9, 0.99, 0.999\}$. Specifically, our method exhibits inferior performance when $q$ is set to 0 compared to other values. This can be attributed to the momentum encoder used for constructing positive and negative sample features losing its momentum update properties, where Eq.~\eqref{update} degenerates into $\bm{\theta}_{\mathrm{m}} = \bm{\theta}$. Consequently, there is a significant reduction in the consistency of feature representations in the queue. This observation is consistent with the findings reported in~\cite{moco}.
\begin{table*}[h]
\caption{Selective risk (\%) of CCL-SC using various momentum coefficient $m$ for various coverage rates (\%) on the CIFAR-100. The means and standard deviations are calculated over 5 trials. The best entries are marked in bold.}\vskip 0.15in
\centering
\begin{tabular}{ccccc}
\toprule 
    Coverage & $q = 0$    & $q = 0.9$  & $q = 0.99$ & $q = 0.999$ \\
    \midrule
100   & 26.72±0.21 & \textbf{26.41±0.22} & 26.55±0.26 & 26.52±0.29 \\
    95    & 23.92±0.30 & 23.56±0.15 & \textbf{23.54±0.15} & 23.64±0.27 \\
    90    & 21.27±0.20 & \textbf{20.91±0.21} & 20.97±0.20 & 20.96±0.32 \\
    85    & 18.71±0.27 & \textbf{18.47±0.17} & 18.57±0.20 & 18.51±0.26 \\
    80    & 16.19±0.31 & \textbf{15.96±0.22} & 16.07±0.15 & 15.97±0.24 \\
    75    & 13.74±0.33 & \textbf{13.53±0.19} & 13.60±0.19 & 13.61±0.30 \\
    70    & 11.33±0.28 & \textbf{11.19±0.18} & 11.23±0.16 & 11.18±0.16 \\
    60    & 6.98±0.15 & 6.94±0.13 & \textbf{6.83±0.15} & 6.97±0.12 \\
    50    & 4.10±0.07 & 4.05±0.17 & \textbf{3.95±0.22} & 4.05±0.20 \\
    40    & 2.34±0.19 & 2.41±0.21 & 2.29±0.33 & \textbf{2.26±0.23} \\
    30    & 1.33±0.18 & 1.30±0.19 & 1.26±0.17 & \textbf{1.23±0.20} \\
    20    & 0.90±0.29 & 0.81±0.16 & 0.71±0.12 & \textbf{0.67±0.14} \\
    10    & 0.60±0.21 & 0.46±0.19 & 0.36±0.08 & \textbf{0.34±0.05} \\ \midrule
    Avg. Rank & 3.92  & 2.00  & 2.15  & \textbf{1.85}  \\
    \bottomrule
    \end{tabular}%

\label{result-m}
\end{table*}

\textbf{Queue size $\bm{s}$}. Tabel~\ref{result-k} shows the performance comparison of our method when maintaining different queue sizes $s \in \{300, 1000, 3000, 10000, 50000\}$. Surprisingly, our method demonstrates performance improvements compared to previous methods, shown in Table~\ref{result-celeba&cifar100}, even when the queue size is set to a remarkably small value, such as 300. Moreover, as the queue size increases, our method exhibits further improvements in performance, particularly in cases with higher coverage.
\begin{table*}[h]
\caption{Selective risk (\%) of CCL-SC using various queue size $s$ for various coverage rates (\%) on the CIFAR-100. The means and standard deviations are calculated over 5 trials. The best entries are marked in bold.}\vskip 0.15in
\centering
\begin{tabular}{cccccc}
\toprule 
    Coverage & $s = 300$  & $s = 1000$ & $s = 3000$ & $s = 10000$ & $s = 50000$ \\
    \midrule
100   & 26.73±0.21 & 26.71±0.18 & 26.55±0.26 & 26.45±0.24 & \textbf{26.08±0.04} \\
    95    & 23.89±0.16 & 23.79±0.19 & 23.54±0.15 & 23.62±0.25 & \textbf{23.22±0.07} \\
    90    & 21.22±0.15 & 21.09±0.29 & 20.97±0.20 & 21.00±0.27 & \textbf{20.66±0.08} \\
    85    & 18.76±0.17 & 18.65±0.31 & 18.57±0.20 & 18.39±0.34 & \textbf{18.19±0.13} \\
    80    & 16.36±0.21 & 16.22±0.37 & 16.07±0.15 & 15.84±0.34 & \textbf{15.73±0.08} \\
    75    & 13.88±0.18 & 13.73±0.32 & 13.60±0.19 & \textbf{13.31±0.42} & 13.42±0.11 \\
    70    & 11.37±0.23 & 11.44±0.24 & 11.23±0.16 & \textbf{11.01±0.33} & 11.09±0.17 \\
    60    & 7.09±0.10 & 7.15±0.18 & \textbf{6.83±0.15} & 6.88±0.24 & 6.94±0.18 \\
    50    & 4.06±0.23 & 4.29±0.16 & \textbf{3.95±0.22} & 4.13±0.21 & 4.15±0.14 \\
    40    & 2.29±0.18 & 2.37±0.16 & 2.29±0.33 & \textbf{2.26±0.24} & 2.44±0.09 \\
    30    & 1.24±0.08 & 1.21±0.14 & 1.26±0.17 & \textbf{1.17±0.07} & 1.42±0.21 \\
    20    & 0.78±0.14 & 0.70±0.20 & 0.71±0.12 & \textbf{0.65±0.08} & 0.75±0.20 \\
    10    & 0.50±0.17 & 0.48±0.19 & \textbf{0.36±0.08} & 0.40±0.09 & 0.48±0.12 \\\midrule
    Avg. Rank & 4.23  & 3.85  & 2.38  & \textbf{1.85}  & 2.54  \\
    \bottomrule
    \end{tabular}%
\label{result-k}
\end{table*}

\textbf{Weight coefficients $\bm{w}$}. Table~\ref{result-weight} presents a comparison of the performance of models trained with varying $w \in \{0.1, 0.5, 1.0, 2.0\}$ applied to the CSC loss. It can be observed that when a relatively larger weight coefficient is assigned, that is, when the CSC loss has a greater impact on model training, the resulting models exhibit better performance, which also confirms the effectiveness of the CSC loss.
\begin{table*}[h]
\caption{Selective risk (\%) of CCL-SC using various weight coefficient $w$ on the CIFAR-100. The means and standard deviations are calculated over 5 trials. The best entries are marked in bold.} \vskip 0.15in
\centering
\begin{tabular}{ccccc}
\toprule 
    Coverage & $w = 0.1$  & $w = 0.5$  & $w = 1.0$  & $w = 2.0$ \\
    \midrule
100   & 26.90±0.04 & 26.59±0.09 & 26.55±0.26 & \textbf{26.25±0.09} \\
    95    & 24.12±0.08 & 23.79±0.13 & 23.54±0.15 & \textbf{23.41±0.16} \\
    90    & 21.44±0.12 & 21.08±0.14 & 20.97±0.20 & \textbf{20.79±0.14} \\
    85    & 19.12±0.15 & 18.68±0.21 & 18.57±0.20 & \textbf{18.34±0.19} \\
    80    & 16.59±0.21 & 16.27±0.26 & 16.07±0.15 & \textbf{15.89±0.19} \\
    75    & 14.24±0.15 & 13.83±0.21 & 13.60±0.19 & \textbf{13.48±0.15} \\
    70    & 11.65±0.21 & 11.39±0.27 & 11.23±0.16 & \textbf{11.12±0.11} \\
    60    & 7.37±0.26 & 7.06±0.13 & \textbf{6.83±0.15} & 7.09±0.15 \\
    50    & 4.20±0.18 & 3.97±0.07 & \textbf{3.95±0.22} & 4.24±0.24 \\
    40    & 2.39±0.15 & 2.34±0.18 & \textbf{2.29±0.33} & 2.39±0.22 \\
    30    & 1.40±0.20 & 1.29±0.18 & \textbf{1.26±0.17} & 1.31±0.12 \\
    20    & 0.96±0.09 & 0.83±0.21 & \textbf{0.71±0.12} & 0.75±0.16 \\
    10    & 0.76±0.08 & 0.48±0.16 & \textbf{0.36±0.08} & 0.50±0.18 \\ \midrule
    Avg. Rank & 3.85  & 2.62  & \textbf{1.54 } & 1.92  \\\bottomrule
    \end{tabular}%
\label{result-weight}
\end{table*}

\textbf{Initial epochs $\bm{E_s}$}. Table~\ref{result-Es} compares the performance of our method when using different initial epochs $E_s \in \{50, 100, 150, 200, 250\}$. It can be observed that our method consistently exhibits stable and robust performance when $E_s$ is set between 50 and 150. The performance tends to deteriorate only if the initial epochs are set too large (i.e., our training mechanism is utilized too late), which leads to insufficient convergence and fluctuations in model performance.
\begin{table*}[h]
\caption{Selective risk (\%) of CCL-SC using various initial epochs $E_s$ for various coverage rates (\%) on the CIFAR-100. The means and standard deviations are calculated over 5 trials. The best entries are marked in bold.}\vskip 0.15in
\centering
\begin{tabular}{cccccc}
\toprule 
    Coverage & $E_s = 50$  & $E_s = 100$ & $E_s = 150$ & $E_s = 200$ & $E_s = 250$ \\
    \midrule
100   & 26.71±0.16 & 26.56±0.18 & \textbf{26.55±0.26} & 26.81±0.24 & 27.13±0.34 \\
    95    & 23.88±0.23 & 23.69±0.15 & \textbf{23.54±0.15} & 23.97±0.23 & 24.25±0.35 \\
    90    & 21.27±0.21 & 21.16±0.08 & \textbf{20.97±0.20} & 21.34±0.17 & 21.65±0.34 \\
    85    & 18.71±0.26 & 18.65±0.08 & \textbf{18.57±0.20} & 18.86±0.23 & 19.20±0.28 \\
    80    & 16.30±0.27 & 16.15±0.12 & \textbf{16.07±0.15} & 16.37±0.21 & 16.77±0.27 \\
    75    & 13.84±0.23 & 13.73±0.16 & \textbf{13.60±0.19} & 13.91±0.15 & 14.40±0.32 \\
    70    & 11.50±0.21 & 11.38±0.29 & \textbf{11.23±0.16} & 11.43±0.27 & 12.02±0.24 \\
    60    & 7.04±0.11 & 7.06±0.21 & \textbf{6.83±0.15} & 7.16±0.24 & 7.55±0.17 \\
    50    & 4.28±0.19 & 4.16±0.12 & \textbf{3.95±0.22} & 4.10±0.17 & 4.59±0.25 \\
    40    & 2.36±0.18 & \textbf{2.20±0.16} & 2.29±0.33 & 2.48±0.36 & 3.13±0.34 \\
    30    & 1.22±0.10 & \textbf{1.18±0.11} & 1.26±0.17 & 1.51±0.24 & 2.22±0.29 \\
    20    & \textbf{0.58±0.12} & 0.79±0.08 & 0.71±0.12 & 0.94±0.20 & 1.57±0.14 \\
    10    & 0.38±0.16 & 0.46±0.16 & \textbf{0.36±0.08} & 0.62±0.21 & 0.94±0.41 \\\midrule
    Avg. Rank &2.77  & 2.15  & \textbf{1.31}  & 3.77  & 5.00  \\
    \bottomrule
    \end{tabular}%

\label{result-Es}
\end{table*}

\section{Further improvement of CCL-SC}\label{improve-sec}
Since CCL-SC operates on the feature representation of the model, it can be seamlessly integrated with existing methods that optimize the model at the classification layer. In this section, we combine CCL-SC with SAT~\cite{NIPS:SAT,SAT} and EM~\cite{Entropy+SR} methods. Specifically, when the current epoch is greater than $E_s$, for a sample $\bm{x}$ with label $y$, we modify the loss function at the classification layer of the model from the cross-entropy loss to the following form:
\begin{align*}
L_{\mathrm{SAT+EM}}=-{t}_{y} \log f_{y}(\bm{x}) - \left(1-{t}_{y}\right) \log f_{(k+1)}(\bm{x}) +\beta ~\mathcal{H}\left(f(\bm{x})\right),
\end{align*}
where $\mathcal{H}$ is the entropy function, and $\beta$ controls the weight of its influence. The training target $\bm{t}$ is dynamically updated using the rule $\bm{t} \leftarrow m_\text{SAT} \cdot \bm{t}+(1-m_\text{SAT}) \cdot f(\bm{x})$, where the momentum term $m_\text{SAT} \in(0,1)$ regulates the weighting of predictions. The first term ${t}_{y} \log f_{y}(\bm{x})$ of $L_{\mathrm{SAT+EM}}$ encourages the model to correctly classify the samples, while the second term $\left(1-{t}_{y}\right) \log f_{(k+1)}(\bm{x})$ encourages the model to abstain from making predictions on samples with low confidence. Due to the combination of the SAT and EM methods, two new hyper-parameters, $m_\text{SAT}$ and $\beta$, are introduced. Here we do not adjust these hyper-parameters but rather directly use the settings for $m_\text{SAT}$ and $\beta$ as introduced in Appendix~\ref{setting sec}.

Table~\ref{result-improvement} and Tabel~\ref{result-improvement-imagenet} present comparisons between the original CCL-SC and the improved CCL-SC (i.e., CCL-SC+SAT+EM) on CIFAR-100 and ImageNet, respectively. The results demonstrate that the improved CCL-SC shows superior performance in selective classification and outperforms the original CCL-SC across all degrees of coverage except 40\% on ImageNet. This discovery shows that CCL-SC not only exhibits superior performance when utilized independently but also highlights high compatibility with other methods to further enhance the performance of selective classification.
\begin{table}[h]
\caption{Selective risk (\%) on the CIFAR-100 dataset for various coverage rates (\%). The means and standard deviations are calculated over 5 trials. The best entries are marked in bold.}\vskip 0.15in
\centering
\begin{tabular}{ccc}

\toprule 
    Coverage & CCL-SC & CCL-SC+SAT+EM\\
    \midrule
    100   & 26.55±0.26 & \textbf{26.41±0.08} \\
    95    & 23.54±0.15 & \textbf{23.51±0.16} \\
    90    & 20.97±0.20 & \textbf{20.85±0.20} \\
    85    & 18.57±0.20 & \textbf{18.31±0.17} \\
    80    & 16.07±0.15 & \textbf{15.79±0.17} \\
    75    & 13.60±0.19 & \textbf{13.41±0.15} \\
    70    & 11.23±0.16 & \textbf{11.06±0.19} \\
    60    & 6.83±0.15 & \textbf{6.72±0.18} \\
    50    & 3.95±0.22 & \textbf{3.67±0.15} \\
    40    & 2.29±0.33 & \textbf{1.97±0.18} \\
    30    & 1.26±0.17 & \textbf{1.05±0.10} \\
    20    & 0.71±0.12 & \textbf{0.49±0.11} \\
    10    & 0.36±0.08 & \textbf{0.26±0.05} \\\bottomrule
    \end{tabular}%
\label{result-improvement}
\end{table}

\begin{table}[h]
\caption{Selective risk (\%) on ImageNet dataset for various coverage rates (\%). The means and standard deviations are calculated over 5 trials. The best entries are marked in bold.}\vskip 0.15in
\centering
\begin{tabular}{ccc}

\toprule 
    Coverage & CCL-SC & CCL-SC+SAT+EM\\
    \midrule
    100   & 26.26±0.10 & \textbf{26.01±0.12} \\
    90    & 20.68±0.07 & \textbf{20.41±0.06} \\
    80    & 15.76±0.07 & \textbf{15.46±0.03} \\
    70    & 11.39±0.10 & \textbf{11.08±0.02} \\
    60    & 7.55±0.09 & \textbf{7.36±0.05} \\
    50    & 4.79±0.04 & \textbf{4.76±0.01} \\
    40    & \textbf{2.95±0.04} & 2.99±0.03 \\
    30    & \textbf{1.83±0.05} & \textbf{1.83±0.06} \\
    20    & 1.22±0.05 & \textbf{1.17±0.07} \\
    10    & 0.72±0.05 & \textbf{0.66±0.07} \\\bottomrule
    \end{tabular}%
\label{result-improvement-imagenet}
\end{table}

\end{document}